\documentclass[english, preprint]{ipsj}

\usepackage[pdftex]{graphicx}
\usepackage[whole]{bxcjkjatype} 
\usepackage[unicode]{hyperref} 

\usepackage{autobreak}
\usepackage{algorithm}
\usepackage{alltt}
\usepackage{amsmath}
\usepackage{amssymb}
\usepackage{amsthm}
\usepackage{array}
\usepackage[noend]{algpseudocode}
\usepackage{bm}
\usepackage{bbm}
\usepackage{booktabs}
\usepackage{cite}
\usepackage{comment}
\usepackage{enumitem} 
\usepackage{etoolbox}
\usepackage{float}

\usepackage{cleveref}

\usepackage{latexsym}
\usepackage{mathtools}
\usepackage{mathrsfs}
\usepackage{nameref}
\usepackage{seqsplit}
\usepackage{textcomp}
\usepackage{times}
\usepackage{xcolor}
\usepackage{url}

\makeatletter
\def\ps@headings{%
  \let\@mkboth\@gobbletwo
  \let\@oddhead\@empty
  \let\@evenhead\@empty
  \let\@oddfoot\@empty
  \let\@evenfoot\@empty
}
\let\ps@IPSJTITLEheadings\ps@headings
\let\ps@empty\ps@headings
\makeatother
\def\sente{\UTF{2617}}
\def\gote{\UTF{2616}}

\def\Underline{\setbox0\hbox\bgroup\let\\\endUnderline}
\def\endUnderline{\vphantom{y}\egroup\smash{\underline{\box0}}\\}
\def\|{\verb|}

\makeatletter
\def\@uketsuke{}
\def\@euketsuke{}
\def\SHUBETUname{}
\makeatother

\makeatletter
\newcommand\newblock{\hskip .11em\@plus.33em\@minus.07em}
\makeatother

\theoremstyle{plain}
\newtheorem{thm}{Theorem}
\newtheorem*{thm*}{Theorem}
\theoremstyle{plain}
\newtheorem{lemma}{Lemma}
\newtheorem*{lemma*}{Lemma}
\theoremstyle{definition}
\newtheorem{dfn}{Definition}
\crefname{thm}{Theorem}{Theorems}

\AtBeginDocument{
\abovedisplayskip     =0.5\abovedisplayskip
\abovedisplayshortskip=0.5\abovedisplayshortskip
\belowdisplayskip     =0.5\belowdisplayskip
\belowdisplayshortskip=0.5\belowdisplayshortskip
}

\algnewcommand\algorithmicinput{{\bfseries\gtfamily Input:}}%
\algnewcommand\algorithmicoutput{{\bfseries\gtfamily Output:}}%
\algnewcommand\AlgInput{\item[\algorithmicinput]}%
\algnewcommand\AlgOutput{\item[\algorithmicoutput]}%
\algrenewcommand\Return{\State\textbf{return} }%

\graphicspath{{images/}}

\title{High-Precision Estimation of the State-Space Complexity of Shogi via the Monte Carlo Method}

\affiliate{utokyo-c}{Graduate School of Arts and Sciences, The University of Tokyo}
\affiliate{utokyo-itc}{Information Technology Center, The University of Tokyo}

\author{Sotaro Ishii}{utokyo-c}[sotaro-ici@tanaka.ecc.u-tokyo.ac.jp]
\author{Tetsuro Tanaka}{utokyo-itc}[ktanaka@g.ecc.u-tokyo.ac.jp]

\begin{document}
\begin{abstract}
	Determining the state-space complexity of Shogi (Japanese Chess) has been a challenging problem, with previous combinatorial estimates leaving a gap of five orders of magnitude ($10^{64}$ to $10^{69}$). This gap arises from the difficulty of distinguishing positions reachable from the initial position among the vast number of board configurations. In this paper, we present a high-precision statistical estimate of the number of reachable Shogi positions. Here, a position consists of the side to move, the board configuration, and the pieces in hand, without move history, and positions related by exchanging the two players or by horizontal reflection of the board are counted once. Our method combines Monte Carlo sampling with a reachability test that performs reverse search toward the set of ``King-King only'' (KK) positions rather than toward the single initial position. Since every KK position and the initial position are mutually reachable, reaching any KK position proves reachability, and removing pieces from the board is a much simpler objective than reconstructing the initial position. Based on a sample of five billion positions, we estimate the number of such positions in Shogi to be $6.55 \times 10^{68}$; both ends of the approximate $3\sigma$ confidence interval round to this value at three significant digits. We also applied this method to Mini Shogi, obtaining approximately $2.38 \times 10^{18}$.
\end{abstract}

\maketitle

\section{Introduction}
Many board games such as Chess, Go, and Shogi (Japanese Chess) are theoretically classified as two-player zero-sum perfect-information games. Numerous programs capable of playing such games at a superhuman level have been developed. Concurrently, research on the theoretical properties of these games has progressed. The primary focus of such research is to determine the ``game-theoretic value'' of a game position---that is, whether optimal play from that position leads to a win for Black, a win for White, or a draw. This process is called ``solving'' the game. Levels of solving a game are classified into the following three categories\cite{allis-1994}.

\begin{description}[itemsep=0pt]
	\item[Ultra-weakly solved] The game-theoretic value of the initial position is known, but the strategy to achieve this value is unknown.
	\item[Weakly solved] The game-theoretic value of the initial position is known, and the optimal moves at each position required to prove it are also known.
	\item[Strongly solved] The game-theoretic value and optimal moves are known for all positions reachable from the initial position. In this context, we say that ``a game position $p$ is reachable from another position $p'$'' if $p$ can be reached from $p'$ by applying a sequence of legal moves.
\end{description}

A game can be strongly solved by either exhaustively enumerating all positions reachable from the initial position, or performing retrograde analysis---a backward search from terminal positions toward the initial position\cite{thompson-retrograde}. However, these methods are infeasible for large-scale games. Therefore, estimating the scale of a game is important for assessing the feasibility of strongly solving it \cite{heule-2007}. One indicator of the ``scale'' of a game is the total number of positions reachable from the initial position, known as the ``state-space complexity''\cite{allis-1994}. State-space complexity has been studied for various board games\cite{shannon, allis-1988, allis-1994, chinchalkar, tromp-go-2006, takeda, compy-2020}.

Shinoda\cite{shinoda} pioneered the estimation of the number of Shogi positions. He showed that the number of Shogi positions lies in the range from $4.65 \times 10^{62}$ to $9.14 \times 10^{69}$. Specifically, he derived the upper-bound by first counting the total number of positions without considering any illegal configurations, and then excluding positions that violate the rules. The lower-bound was obtained by constructing specific patterns of reachable positions and counting those positions that match the patterns. Subsequently, Miyako et al.\cite{miyako} improved upon Shinoda's method and derived the range $2.45 \times 10^{64}$ to $6.78 \times 10^{69}$.

The bounds obtained in \cite{shinoda} and \cite{miyako} have a gap spanning more than five orders of magnitude. A major factor contributing to this large gap is that the lower-bound estimation only counts positions matching specific board configuration patterns, thereby considering only a small fraction of all positions. For games with few constraints on the legality of moves, it is relatively easy to obtain estimates close to the exact number of positions by enumerating board configurations. In contrast, in Shogi, the outcome is determined by the absence of legal moves to evade a check, and leaving one's King in check is illegal. These factors make it difficult to accurately count positions satisfying the condition of being reachable from the initial position by playing only legal moves.

For games where exact enumeration of positions is difficult, a method has been proposed that generates a large number of positions whose reachability is unknown, examines how many of them are reachable, and thereby estimates the number of positions statistically\cite{shinoda, tromp-cpr}. To apply this method, an algorithm for determining the reachability of a given position from the initial position is required. In this study, we propose a method that determines reachability by performing a backward best-first search from a generated position toward a set of positions satisfying specific conditions. When our proposed method was applied to standard Shogi and Mini Shogi\cite{umebayashi}, reachability from the initial position could be determined with relatively low search effort. As a result, we estimated the number of canonical reachable positions in standard Shogi at $6.55 \times 10^{68}$ and the corresponding number in Mini Shogi at $2.38 \times 10^{18}$, both with three significant digits.

Throughout this paper, a position consists of the side to move, the board configuration, and the pieces in hand; move history, including the number of repetitions, is not part of a position. We count only positions with Black to move, because every White-to-move position corresponds to a Black-to-move position obtained by rotating the board 180 degrees about its center and exchanging the two players, and one is reachable if and only if the other is. We also identify each position with its horizontal reflection; although the initial position is not left-right symmetric, a position is reachable if and only if its reflection is. We call the representatives chosen in this way canonical positions, and the numbers above are the numbers of canonical reachable positions (see \cref{chapter:canonical}). Under the definition of previous studies\cite{shinoda, miyako}, which distinguishes horizontally reflected positions, our result corresponds to approximately $1.31 \times 10^{69}$ positions.

\section{Shogi}

\subsection{Rules of Shogi}
Shogi is a two-player zero-sum perfect-information game widely played in Japan, in which the two players alternately make moves starting from the initial configuration shown in \figref{shogi-initpos}. We describe the rules relevant to this study, referring to the competition rules of the Japan Shogi Association\cite{shogi-rules}.

\begin{figure}[tb]
	\centering
	\begin{minipage}[c]{\columnwidth}
		\centering
		\includegraphics[width=0.85\columnwidth]{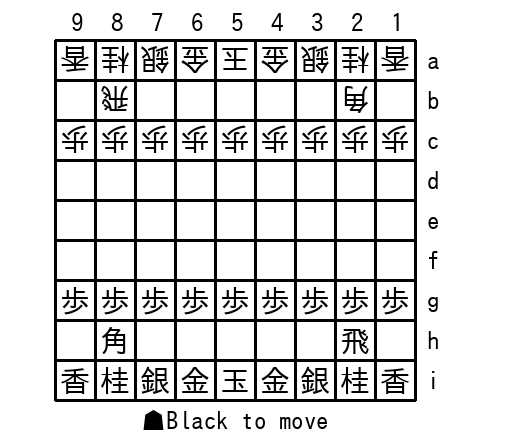}
	\end{minipage}
	\caption{The initial position of Shogi}
	\label{shogi-initpos}
\end{figure}

\begin{itemize}[itemsep=0cm]
	\item The player who moves first is called Black (先手; \emph{Sente}) and the one who moves second is called White (後手; \emph{Gote}). This is opposite to the naming convention in Chess.
	\item The game uses 40 pieces across 8 types: King (玉), Rook (飛), Bishop (角), Gold (金), Silver  (銀), Knight (桂), Lance (香), and Pawn (歩).
	\item Captured enemy pieces become the capturing player's pieces in hand. The player may drop them on any vacant square on the board, provided that the drop does not violate the rules described below.
	\item Rook, Bishop, Silver, Knight, Lance, and Pawn pieces can be promoted when entering or leaving the opponent's three ranks, transforming into Dragon (龍), Horse (馬), Promoted Silver (全), Promoted Knight (圭), Promoted Lance (杏), and Tokin (と), respectively.
	\item When a piece satisfies the conditions for promotion, the player may choose whether to promote it. However, if declining promotion results in a piece with no legal moves (described below), then promotion is mandatory.
	\item ``Checkmate'' occurs when no move can prevent the King from being captured on the next turn. The game ends when this situation arises.
	\item A player must make a move on their turn. Passing is not permitted, even when no legal move is available. In Chess, such a situation where no legal move exists and the player is not in check is called ``stalemate'' and results in a draw. In Shogi, there are no specific rules regarding stalemate, but it is generally considered a loss for the player whose King is stalemated.
	\item When the same tuple of side to move, board configuration, and pieces in hand appears four times during a single game, it is called ``repetition'' (\emph{Sennichite}). Under common convention, repetition results in a draw. Under the Japan Shogi Association's competition rules, repetition results in no contest, and the game is replayed with Black and White swapped.
	\item The following actions constitute rule violations (illegal moves):
	      \begin{description}
		      \item[Two Pawns] Placing a second unpromoted Pawn of the same player on a file (column) that already contains one.
		      \item[Piece with no legal moves] Placing one's own piece in a position from which it can never move for the rest of the game.
		      \item[Drop Pawn Mate] Checkmating the opponent's King by dropping a Pawn from hand. Some conventions distinguish ``Drop Pawn Stalemate'' (where a Pawn drop stalemates the opponent without giving check) from Drop Pawn Mate, not treating the former as illegal. However, in this study, we treat it as Drop Pawn Mate\footnote{As described later, the difference in this definition had almost no effect on the estimation of the number of positions.}.
		      \item[Suicide move / Ignoring a check] Making a move that allows one's own King to be captured, or failing to evade a check.
		      \item[Perpetual check] A special case of repetition (\emph{Sennichite}) in which all of one player's moves are checks. When this occurs, the player who gave the perpetual check loses.
	      \end{description}
\end{itemize}

\subsection{Mini Shogi}
Mini Shogi\cite{umebayashi}, devised by Shigenobu Kusumoto in 1970, is played on a $5 \times 5$ board with a reduced number and variety of pieces. As with standard Shogi, the objective is to checkmate the opponent's King, and the two players alternately make moves starting from the initial configuration shown in \figref{minishogi-initpos}. The differences between Mini Shogi and standard Shogi are as follows.

\begin{figure}[tb]
	\centering
	\begin{minipage}[c]{\columnwidth}
		\centering
		\includegraphics[width=0.85\columnwidth]{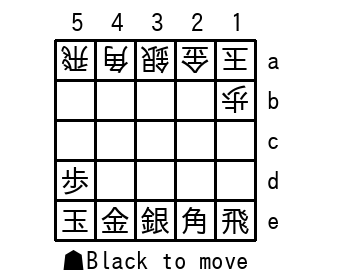}
	\end{minipage}
	\caption{The initial position of Mini Shogi}
	\label{minishogi-initpos}
\end{figure}

\begin{itemize}[itemsep=0cm]
	\item The game uses 12 pieces across 6 types: King, Rook, Bishop, Gold, Silver, and Pawn.
	\item Rook, Bishop, Silver, and Pawn pieces can be promoted when entering or leaving the opponent's first rank, transforming into Dragon, Horse, Promoted Silver, and Tokin, respectively.
	\item In actual play, the rule ``repetition results in a win for White'' is often added to encourage Black to actively avoid repetition\cite{mizuta}.
\end{itemize}
Other rules follow those of standard Shogi.

\subsection{Definition of Position}\label{chapter:def-state}
\subsubsection{Representation of Positions}
In this study, we represent a Shogi position as a tuple consisting of the side to move, board configuration, and pieces in hand. Specifically, let $t \in \left\{ \mathrm{B}, \mathrm{W} \right\}$ ($\mathrm{B}$: Black, $\mathrm{W}$: White) denote the current side to move, $\bm{b}$ denote the board configuration, and $\bm{h}$ denote the pieces in hand. We define a Shogi position $p$ as the following 3-tuple.
\begin{equation}\label{def:shogi-pos}
	p \coloneq \left( t,\bm{b},\bm{h} \right).
\end{equation}
In the discussion of Shogi's state-space complexity, we consider only the above three components of information as the state and do not consider other information (such as the history of moves from the initial position). We justify this choice in \cref{chapter:excluded-rules}.

We use the following terms in this paper. A \emph{position} is a tuple \eqref{def:shogi-pos} in which each player has exactly one King on the board and the pieces on the board and in hand together form the piece set of the game; it may violate the rules. A position is \emph{reachable} if it can be reached from the initial position by a sequence of legal moves, and \emph{unreachable} otherwise. A position that violates one of the local rules described in \cref{chapter:shogi-remove-illegal}, such as Two Pawns, is always unreachable, but not every unreachable position violates these rules. The \emph{canonical} and \emph{K-canonical} positions defined below concern only the choice of representatives under symmetry and are unrelated to reachability. The positions from which we sample are called \emph{candidate positions}.

\subsubsection{Rules Not Reflected in Positions}\label{chapter:excluded-rules}
Under Shogi's repetition (\emph{Sennichite}) rule, even for the same tuple of side to move, board configuration, and pieces in hand, the game-theoretic value may differ depending on how many times that tuple has appeared during play from the initial position. For example, consider a case where playing move $m$ from position $A$ leads to position $B$ and results in a win, while all other moves lead to a loss. If the number of times position $B$ has appeared from the initial position is at most 2, then one can play $m$ and win. However, if the count is 3, then the best strategy is to play $m$ and force a repetition, resulting in a draw (or a rematch). Nevertheless, suppose that the purpose of ``solving Shogi'' is limited to determining game-theoretic values under the assumption that each position appears at most once during play from the initial position. In that case, case analysis based on the number of occurrences becomes unnecessary. Therefore, defining a position solely by the tuple \eqref{def:shogi-pos} is sufficient. Accordingly, we define reachability by legal moves alone and ignore history-dependent rules such as repetition and perpetual check. For the same reason, the Mini Shogi convention that repetition is a win for White is not considered.

Moreover, the official competition rules of the Japan Shogi Association\cite{shogi-rules} stipulate that a game reaching 500 moves results in no contest. In this study, we adopt more general rules and do not consider move count limits.

If the Entering King rule were taken into account, there could be cases where the symmetry between Black and White breaks down\footnote{``Entering King (入玉; \emph{Nyugyoku})'' refers to the King entering the opponent's promotion zone. When both Kings have entered the promotion zone, and it is difficult to reach a checkmate, the outcome is determined by counting the points of both players' pieces based on mutual agreement. Under the ``24-point rule'' adopted in professional tournaments, Black and White are symmetric, whereas under the ``27-point rule'' used in amateur tournaments, the victory conditions differ between Black and White, breaking the symmetry.}. However, since a player may continue the game without declaring even when the winning conditions are met, incorporating the Entering King declaration into the position count has little significance. In this study, we do not consider the Entering King declaration rule.

\subsubsection{Canonicalization of Positions}\label{chapter:canonical}
When analyzing a game, positions that can be regarded as equivalent due to symmetry are often treated as a single position. In this study, we refer to this operation as ``canonicalization of positions.'' Let $\operatorname{Cswap}(p)$ denote the position obtained from $p$ by exchanging the roles of Black and White; that is, the board is rotated 180 degrees about its center, and the ownership of every piece, the pieces in hand, and the side to move are all exchanged. Let $\operatorname{Hflip}(p)$ denote the position obtained by horizontally mirroring the board of $p$, without changing the pieces in hand or the side to move. The game-theoretic value of $\operatorname{Cswap}(p)$ is that of $p$ with the roles of the players exchanged, and since all piece movements in Shogi are left-right symmetric, the game-theoretic value of $\operatorname{Hflip}(p)$ is the same as that of $p$. For counting reachable positions, however, we also need these transformations to preserve reachability. This does not follow from the symmetry of the rules alone, because neither transformation maps the initial position to itself: $\operatorname{Cswap}$ gives White the move, and $\operatorname{Hflip}$ exchanges the files of the Rook and the Bishop. Both transformations map every legal move to a legal move and map KK positions (\cref{def:kk}) to KK positions. Using these facts together with the mutual reachability of KK positions and the initial position (\cref{thm:initial-kk-each-other}), we obtain the following.
\begin{thm}\label{thm:symmetry}
	For any position $p$, the following are equivalent: $p$ is reachable; $\operatorname{Cswap}(p)$ is reachable; $\operatorname{Hflip}(p)$ is reachable.
\end{thm}
\begin{proof}
	Let $G$ be either $\operatorname{Cswap}$ or $\operatorname{Hflip}$, and suppose that $p$ is reachable. By \cref{thm:kk-unreachable}, $p$ is reachable from some KK position $q$. Applying $G$ to every position in a move sequence from $q$ to $p$ gives a legal move sequence from the KK position $G(q)$ to $G(p)$. Since $G(q)$ is reachable from the initial position by \cref{thm:initial-kk-each-other}, so is $G(p)$. The converse follows from $G(G(p)) = p$.
\end{proof}
By \cref{thm:symmetry}, $\operatorname{Cswap}$ gives a one-to-one correspondence between reachable Black-to-move positions and reachable White-to-move positions. The number of reachable positions is therefore exactly twice the number of reachable Black-to-move positions, and it suffices to consider only positions where it is Black's turn. Furthermore, a Black-to-move position and its horizontal mirror image are either both reachable or both unreachable. Assuming that an order relation is defined on the set of positions (\cref{chapter:position-order}), we call a Black-to-move position $p$ a \emph{canonical position} if $p \leq \operatorname{Hflip}(p)$, and count only canonical positions. A position with $p = \operatorname{Hflip}(p)$ is canonical and is counted once.

Checking canonicality requires comparing whole positions, which is inconvenient for enumerating candidate positions. We therefore also use a weaker condition that looks only at the two Kings. A Black-to-move position $p$ is called a \emph{K-canonical position} if the King placement of $p$ is no greater than that of $\operatorname{Hflip}(p)$; concretely, as seen from Black, either Black's King is on a file to the left of the central file, or Black's King is on the central file and White's King is on the central file or to its left. Since the order on positions compares the Kings first, every canonical position is K-canonical. Conversely, a K-canonical position whose King placement differs from that of its horizontal reflection is also canonical; only when both Kings lie on the central file do the remaining pieces decide which of the two reflected positions is canonical. Thus, K-canonicality is a convenient condition for enumeration, whereas canonicality specifies the final counting convention. The formal definition is given in \cref{chapter:k-canonical}.

Note that previous studies on the number of Shogi positions\cite{shinoda, miyako} also identified $p$ with $\operatorname{Cswap}(p)$, but distinguished horizontally mirrored positions. However, given that the game-theoretic value of a horizontally mirrored position is the same, we believe our definition is more appropriate for characterizing the properties of the game. The estimate under the definition of the previous studies is given in \cref{chapter:exp-shogi}.

\section{Statistical Estimation of the Number of Positions and Enumeration of Candidate Positions}

\subsection{Principle of Statistical Estimation of the Number of Positions} \label{chapter:estimate-method}
As discussed in the Introduction, for games such as Connect Four\cite{allis-1988} and Go\cite{tromp-go-2006}, which have few constraints on move legality, estimates close to the exact number of positions can be obtained relatively easily by enumerating positions. However, in the case of Shogi, it is difficult to exactly determine the number of positions satisfying the condition ``reachable from the initial position by playing only legal moves under the rules'' within a practical amount of time.

\begin{figure}[tb]
	\centering
	\includegraphics[width=1\columnwidth]{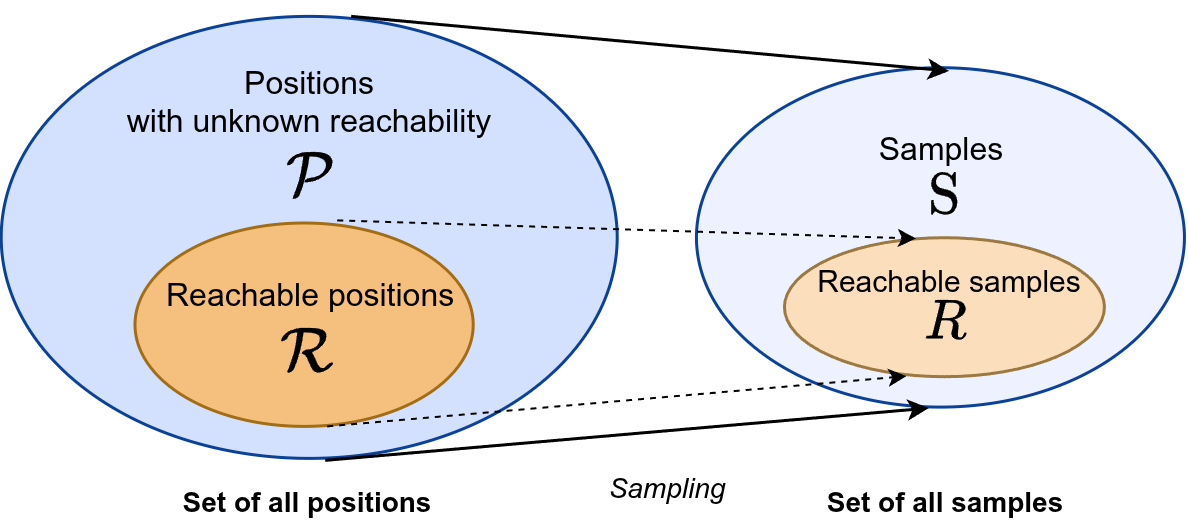}
	\caption{Schematic of statistical estimation. In this study, the population $\mathcal P$ shown in the diagram is $\mathcal P_K$, and the target $\mathcal R$ consists of canonical reachable positions. The accepted sampled positions $R$ are those with $X_j=1$; all sampled positions in $S$ contribute to the denominator.}
	\label{model}
\end{figure}

Therefore, instead of determining the exact number of Shogi positions, we estimate it with high precision using a Monte Carlo method. Let $\mathcal P_K$ be the set of K-canonical positions (\cref{chapter:canonical}), which serve as candidate positions; its cardinality is computed exactly in the next subsection. Let $\mathcal R$ be its subset consisting of positions that are both canonical and reachable from the initial position. Since every canonical position is K-canonical, $\mathcal R$ contains all canonical reachable positions.

We sample $n$ positions $p_1,\ldots,p_n$ uniformly from $\mathcal P_K$. Each sampled position is obtained by generating a rank uniformly at random, independently of the others, and unranking it (\cref{chapter:position-order}); that is, we sample with replacement. Since $n$ is negligible compared with $|\mathcal P_K|$, the expected number of pairs of identical sampled positions is below $10^{-3}$ for Mini Shogi and $10^{-50}$ for Shogi, so sampling without replacement would give practically the same result. For each sampled position, we set $X_j=1$ if $p_j$ is canonical and reachable, and $X_j=0$ otherwise. We first apply the canonicality filter, followed by local rule filters and, when necessary, reverse search toward KK positions. We write $S=(p_1,\ldots,p_n)$ and $|S|=n$. The estimator is
\begin{equation}\label{eq:count-estimator}
	\widehat{|\mathcal R|}=|\mathcal P_K|\hat p,
	\qquad \hat p=\frac{1}{n}\sum_{j=1}^{n}X_j.
\end{equation}
The denominator $n$ includes all sampled positions, including those rejected by the canonicality filter. A non-canonical sampled position is outside the target set, but need not be unreachable. A position that is invariant under horizontal reflection is canonical and is counted once.

The population proportion is $p_{\mathcal R}=|\mathcal R|/|\mathcal P_K|$. Under independent uniform sampling, an approximate $3\sigma$ confidence interval based on the normal approximation to the binomial distribution is
\begin{equation}\label{pr-interval}
	\left[\hat p-3\sqrt{\frac{\hat p(1-\hat p)}{n}},\;
		\hat p+3\sqrt{\frac{\hat p(1-\hat p)}{n}}\right].
\end{equation}
Multiplying the endpoints by $|\mathcal P_K|$ gives an approximate confidence interval for $|\mathcal R|$. This interval accounts only for sampling error, assuming that each sampled position is classified correctly. \figref{model} illustrates the estimation principle.

\subsection{Enumeration of Candidate Positions}\label{chapter:position-order}
To execute the method described in \cref{chapter:estimate-method}, we need to count the number of elements in $\mathcal{P}_K$ and sample elements from $\mathcal{P}_K$ uniformly at random. We sample from $\mathcal{P}_K$ using ranking and unranking\cite{kreher1998combinatorial}. Ranking is the process of assigning a unique integer (rank) in the range $\{0,1,2,\dots,|\mathcal{P}_K|-1\}$ to each element of  $\mathcal{P}_K$. The inverse operation is called unranking.

In this section, we describe the definition of $\mathcal{P}_K$ for Shogi and the method for counting $|\mathcal{P}_K|$. Using the board dimensions and piece inventory of Mini Shogi, the method described in this section can also be used to compute the number of candidate positions for Mini Shogi.


\begin{figure}[tb]
	\centering
	\resizebox{\columnwidth}{!}{
		\begin{tabular}{|*{9}{c|}}\hline
			(0,0) & (1,0) & (2,0) & (3,0) & (4,0) & (5,0) & (6,0) & (7,0) & (8,0) \\\hline
			(0,1) & (1,1) & (2,1) & (3,1) & (4,1) & (5,1) & (6,1) & (7,1) & (8,1) \\\hline
			(0,2) & (1,2) & (2,2) & (3,2) & (4,2) & (5,2) & (6,2) & (7,2) & (8,2) \\\hline
			(0,3) & (1,3) & (2,3) & (3,3) & (4,3) & (5,3) & (6,3) & (7,3) & (8,3) \\\hline
			(0,4) & (1,4) & (2,4) & (3,4) & (4,4) & (5,4) & (6,4) & (7,4) & (8,4) \\\hline
			(0,5) & (1,5) & (2,5) & (3,5) & (4,5) & (5,5) & (6,5) & (7,5) & (8,5) \\\hline
			(0,6) & (1,6) & (2,6) & (3,6) & (4,6) & (5,6) & (6,6) & (7,6) & (8,6) \\\hline
			(0,7) & (1,7) & (2,7) & (3,7) & (4,7) & (5,7) & (6,7) & (7,7) & (8,7) \\\hline
			(0,8) & (1,8) & (2,8) & (3,8) & (4,8) & (5,8) & (6,8) & (7,8) & (8,8) \\\hline
		\end{tabular}
	}
	\caption{Coordinates $(x,y)$ on the board}
	\label{board-coordinates}
\end{figure}

We define the $(x,y)$ coordinates of the Shogi board as seen from Black's perspective, as shown in \figref{board-coordinates}, and define the order relation on two distinct squares $(x,y), (x',y')$ as follows.
\begin{align}
	(x, y) < (x', y') \Longleftrightarrow(x < x') \: \lor \:(x = x' \: \land \: y < y').
\end{align}

We use $i \in \left\{ 0,1,\ldots,7 \right\}$ for the type of piece before promotion (hereafter called ``basic piece type''), where $0,1,\ldots,7$ correspond to King, Gold, Knight, Lance, Pawn, Silver, Rook, and Bishop, respectively. Let $s_{i}$ denote the total number of a basic piece type $i$. From the rules, we have $\left( s_{0},s_{1},\ldots,s_{7} \right) = (2,4,4,4,18,4,2,2)$.

First, we determine the number of pieces in hand. We denote the two players as $\left\{ \mathrm{B},\mathrm{W} \right\}$, representing Black and White, respectively. For basic piece type $i$, let $h_{i,\mathrm{B}}$ be the number of pieces of type $i$ held by Black and $h_{i,\mathrm{W}}$ be the number held by White. We define the following (we include $i = 0$ for notational convenience, although Kings cannot be held in hand):
\begin{align}
	\bm{h}_{i} \coloneq \left( h_{i,\mathrm{B}},h_{i,\mathrm{W}} \right), \: \bm{h} \coloneq \left( \bm{h}_{0},\bm{h}_{1},\ldots,\bm{h}_{7} \right).
\end{align}

Furthermore, we express the total number $u_{i}$ of pieces of basic piece type $i$ held in hand as follows.
\begin{equation}\label{eq:hand_pieces}
	u_{i} \coloneq h_{i,\mathrm{B}} + h_{i,\mathrm{W}}, \: \bm{u} \coloneq (u_{0},u_{1},\ldots,u_{7}).
\end{equation}

Next, we describe the representation of pieces on the board. Pieces on the board can be classified according to two attributes---ownership (``Black/White'') and promotion status (``promoted/unpromoted.'')---yielding four categories. We denote promotion status by $\left\{ + , - \right\}$ ($+$: promoted, $-$: unpromoted). For owner $t \in \left\{ \mathrm{B},\mathrm{W} \right\}$ and $\sigma \in \left\{ + , - \right\}$, we define $B_{i,t,\sigma}$ as the set of squares occupied by pieces of basic piece type $i$ with these attributes, and write its cardinality as $\left| B_{i,t,\sigma} \right|$.

Let $\bm{B}_{i}$ denote the tuple combining the four attributes ``Black/White'' and ``promoted/unpromoted'' for each basic piece type $i$, and let $\bm{B}$ represent the entire board configuration as a tuple of each $\bm{B}_{i}$\footnote{King and Gold, which have no promoted forms, are also represented as 4-tuples. The elements corresponding to promoted pieces are always empty lists.}.
\begin{gather}
	\bm{B}_{i} \coloneq \left( B_{i,\mathrm{B}, +},B_{i,\mathrm{W}, +},B_{i,\mathrm{B}, -},B_{i,\mathrm{W}, -} \right) \notag \\
	\bm{B} \coloneq \left( \bm{B_{0}},\bm{B_{1}},\ldots,\bm{B_{7}} \right).
\end{gather}

\begin{figure}[tb]
	\centering
	\begin{minipage}[c]{\columnwidth}
		\centering
		\includegraphics[width=0.85\columnwidth]{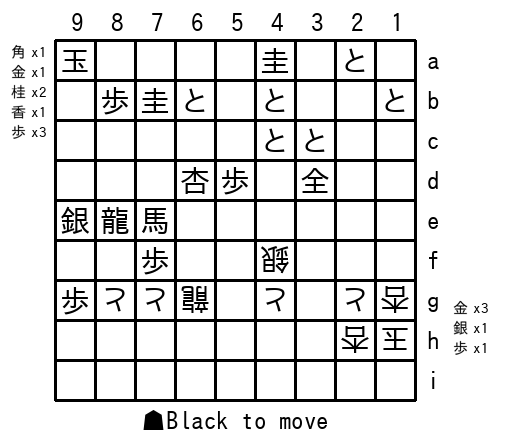}
	\end{minipage}
	\caption{The final position of Game 4 of the 2nd Den-O Sen.}
	\label{puella-tsukada}
\end{figure}

For example, considering the Silver (basic piece type $i = 5$) on the board in \figref{puella-tsukada}: Black's Promoted Silver is at $(6,3)$, Black's Silver is at $(0,4)$, and White's Silver is at $(5,5)$. Thus, the representation is as follows.
\begin{gather}
	B_{5,\mathrm{B}, +} =\{ (6,3) \},\: B_{5,\mathrm{W}, +} = \emptyset,\notag \\
	B_{5,\mathrm{B}, -} =\{ (0,4) \},\: B_{5,\mathrm{W}, -} =\{ (5,5) \} \notag \\
	\therefore\bm{B}_{5} = ( \:\{ (6,3) \},\emptyset,\{ (0,4) \}, \{ (5,5) \}\: ).
\end{gather}

We now define the lexicographic order used to compare a position with its horizontal reflection, as implemented in our experimental code. Each square $(x,y)$ is assigned the index $9x+y$ ($5x+y$ for Mini Shogi), which is consistent with the order on squares defined above. Pieces are ordered first by basic piece type $i$, then promoted before unpromoted, and then Black before White; that is, in the order of the components of $\bm{B}_{i}$. The board of a position is represented as the list of pairs (piece, square index) for all pieces on the board, sorted in ascending order, and two positions are compared by comparing these lists lexicographically. The pieces in hand and the side to move are not compared, since they are identical in a Black-to-move position and its horizontal reflection. Because a position and its reflection contain the same pieces, this comparison is equivalent to comparing the sets $B_{i,t,\sigma}$, each written as an increasing list of squares, in increasing order of $i$ and, for each $i$, in the order $(\mathrm{B},+)$, $(\mathrm{W},+)$, $(\mathrm{B},-)$, $(\mathrm{W},-)$.

For example, if Black's unpromoted Golds are on $(6,5)$ and $(8,0)$, the corresponding part of the list is $[(\text{Gold},59),(\text{Gold},72)]$. After horizontal reflection, these Golds move to $(2,5)$ and $(0,0)$, giving $[(\text{Gold},0),(\text{Gold},23)]$, which is smaller. This comparison is used in the example in \cref{chapter:shogi-remove-flip}.

\begin{figure*}[tb]
	\centering
	\includegraphics[width=2\columnwidth]{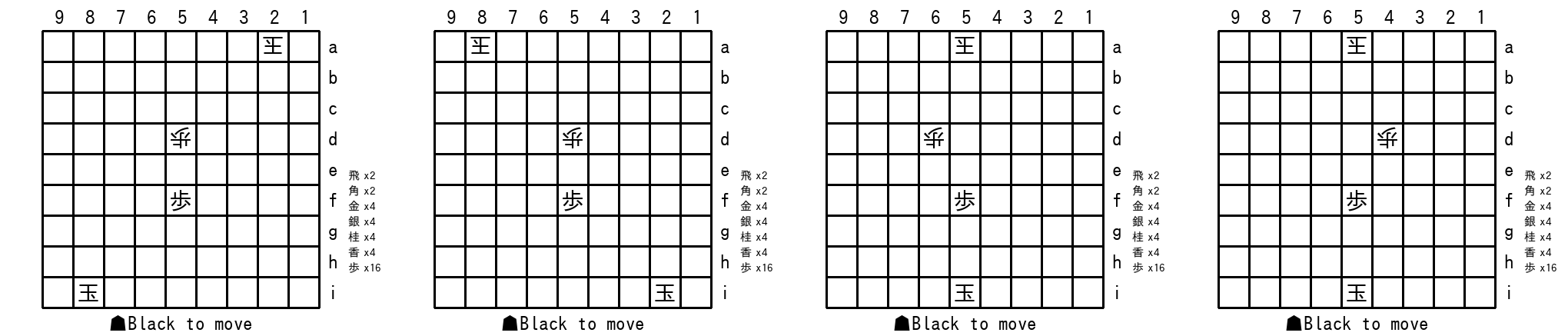}
	\caption{K-canonical positions and non-K-canonical positions}
	\label{kregular}
\end{figure*}

\subsubsection{K-Canonical Positions}\label{chapter:k-canonical}
Since the King cannot be promoted, the coordinates of Black's King and White's King in position $p$ are uniquely contained in $B_{0,\mathrm{B}, -}$ and $B_{0,\mathrm{W}, -}$, respectively. We define the King placement of position $p$ as
\begin{equation}
	\kappa(p) \coloneq \left( B_{0,\mathrm{B}, -},\:  B_{0,\mathrm{W}, -}  \right).
\end{equation}
Using this notation, we formally define the K-canonical position introduced in \cref{chapter:canonical}.

\begin{dfn}[K-canonical position]
	Using the lexicographic order defined above, we call $p$ a ``K-canonical position'' when it satisfies the following conditions.
	\begin{itemize}[itemsep=0cm]
		\item $p$ is a Black-to-move position.
		\item $\kappa(p) \le \kappa(\operatorname{Hflip}(p))$ holds. That is, the King placement $\kappa(p)$ of $p$ is no greater than the King placement $\kappa(\operatorname{Hflip}(p))$ of the horizontally mirrored position $\operatorname{Hflip}(p)$.
	\end{itemize}
\end{dfn}

Let $p_{1}$ and $p_{2}$ denote the first and second positions from the left in \figref{kregular}, respectively. Then
\begin{equation}\begin{array}{r}
		\bm{B}_{0}\left( p_{1} \right) = \left( \emptyset,\emptyset,\{ (1,8) \},\{ (7,0) \} \right), \\
		\bm{B}_{0}\left( p_{2} \right) = \left( \emptyset,\emptyset,\{ (7,8) \},\{ (1,0) \} \right).
	\end{array}
\end{equation}

Since $\kappa(p_{1}) < \kappa(p_{2})$ holds, $p_{1}$ is a K-canonical position and $p_{2}$ is not. Similarly, the third position $p_{3}$ and fourth position $p_{4}$ from the left in \figref{kregular} are horizontal reflections of each other. Furthermore, since their King placements are equal, as shown below, both are K-canonical positions. However, since $p_{3} < p_{4}$, $p_{3}$ is a canonical position while $p_{4}$ is not.
\begin{equation}
	\bm{B}_{0}\left( p_{3} \right) = \bm{B}_{0}\left( p_{4} \right) = \left( \emptyset,\emptyset, \{ (4,8) \},\{ (4,0) \} \right).
\end{equation}

We denote the set of all K-canonical positions of Shogi by $\mathcal{P}_{K}$, and compute its cardinality $\left| \mathcal{P}_{K} \right|$ later.

\subsubsection{Total Number of Pieces in Hand and on the Board}
For the total number of pieces in hand $u_{i}$ defined in \eqref{eq:hand_pieces}, since Kings cannot be held in hand, $u_{0} = 0$ always holds, and $0 \leq u_{i} \leq s_{i}(1 \leq i \leq 7)$. Moreover, the number of pieces of basic piece type $i$ on the board satisfies the following.
\begin{equation}\label{b_and_h}
	\sum_{t \in \left\{ \mathrm{B},\mathrm{W} \right\}}\sum_{\sigma \in \left\{ + , - \right\}}\left| B_{i,t,\sigma} \right| = s_{i} - u_{i}\quad(0 \leq i \leq 7).
\end{equation}

Since both Kings must always be present on the board and Kings and Golds have no promoted forms, the following holds.
\begin{gather}
	|B_{0,\mathrm{B},+}|=|B_{0,\mathrm{W},+}|=0,\notag\\
	|B_{0,\mathrm{B},-}|=|B_{0,\mathrm{W},-}|=1,\notag\\
	|B_{i,t,+}|=0\quad(0\leq i\leq1,\;t\in\{\mathrm B,\mathrm W\}).
\end{gather}

We represent the number of pieces on the board for each attribute of basic piece type $i$ as the vector $\bm{n}_{i}$.
\begin{equation}
	\bm{n}_{i} \coloneq \left( \left| B_{i,\mathrm{B}, +} \right|,\left| B_{i,\mathrm{W}, +} \right|,\left| B_{i,\mathrm{B}, -} \right|,\left| B_{i,\mathrm{W}, -} \right| \right).
\end{equation}

When the number of pieces of basic piece type $i$ on the board is $n$, we define the set $\mathcal{N}_{i,n}$ of possible values of $\bm{n}_{i}$ as follows. All tuple components below are nonnegative integers.
\begin{equation}\label{cal_n_def}
	\left\{
	\begin{aligned}
		\mathcal{N}_{0,2} & = \left\{ (0,0,1,1) \right\},                     \\
		\mathcal{N}_{1,n} & = \left\{ \left( 0,0,n_{2},n_{3} \right) \middle|
		\begin{gathered}
			0 \leq n_{2} \leq n,\\ 0 \leq n_{3} \leq n, \\ n_{2} + n_{3} = n ,
		\end{gathered} \right\}     \\
		\mathcal{N}_{i,n} & = \left\{(n_0,n_1,n_2,n_3)\;\middle|\;
		\sum_{j=0}^{3}n_j=n\right\}\ (i>1).
	\end{aligned}
	\right.
\end{equation}

\subsubsection{Number of Board Configurations and Piece-in-Hand Combinations}\label{chapter:board-configuration}
Given $\bm{u}$, the total number of hand-piece distributions for each player $N_{\text{hand}}\left( \bm{u} \right)$ can be computed as follows.
\begin{equation}\label{n_hand_h_eq}
	N_{\text{hand}}\left( \bm{u} \right) = \prod_{i = 1}^{7}\left( u_{i} + 1 \right).
\end{equation}

We compute the total number of board configurations $N_{\text{board}}\left( \bm{u} \right)$ corresponding to $\bm{u}$. From \eqref{b_and_h}, once $\bm{h}$ is determined, the number of pieces on the board is also determined. First, we define the total number of arrangements $N_{\text{comb}}\left( e,\bm{n}_{i} \right)$ when placing pieces of basic piece type $i$ with per-attribute board piece counts $\bm{n}_{i} = \left( n_{0},n_{1},n_{2},n_{3} \right)$ on a board with $e$ empty squares.
\begin{multline}
	N_{\text{comb}}\left(e,\bm n_i\right)
	\coloneq \binom{e}{n_0}\binom{e-n_0}{n_1}\\
	{}\times\binom{e-n_0-n_1}{n_2}\binom{e-n_0-n_1-n_2}{n_3}.
\end{multline}

We write the number of ways to place pieces of basic piece type $i$ on the board as $N_{\text{btype}}\left( \bm{u},i \right)$. For the King ($i = 0$), by the definition of K-canonical positions, the $x$ coordinate of Black's King is restricted to $0 \leq x \leq 4$. When Black's King is on one of the $4 \times 9$ squares satisfying $0 \leq x \leq 3$, White's King can be placed on $9 \times 9 - 1$ squares; when Black's King is on one of the $9$ squares satisfying $x = 4$, White's King can be placed on $5 \times 9 - 1$ squares. Note that this value does not depend on $\bm{u}$.
\begin{align}
	N_{\text{btype}}(\bm{u},0) & = 4 \times 9 \times (9 \times 9 - 1) + 1 \times 9 \times (5 \times 9 - 1) \notag \\
	                           & = 3276.
\end{align}

For basic piece types $i \geq 1$, the number of pieces on the board is $s_{i} - u_{i}$. Therefore
\begin{multline}\label{n_btype_comb_def}
	N_{\text{btype}}\left(\bm u,i\right)\\
	=\sum_{\bm n_i\in\mathcal N_{i,s_i-u_i}}
	N_{\text{comb}}\left(81-\sum_{i'<i}(s_{i'}-u_{i'}),\bm n_i\right).
\end{multline}
From the above, we have
\begin{equation}
	N_{\text{board}}\left( \bm{u} \right) \coloneq \prod_{i = 0}^{7}N_{\text{btype}}\left( \bm{u},i \right).
\end{equation}

\subsubsection{Total Number of Elements}
The total number of K-canonical positions with hand-piece totals $\bm{u}$, denoted by $N\left( \bm{u} \right)$, is
\begin{equation}
	N\left( \bm{u} \right) = N_{\text{hand }}\left( \bm{u} \right) \cdot N_{\text{board }}\left( \bm{u} \right).
\end{equation}

The set of possible hand-piece totals $\bm{u}$ is
\begin{equation}
	\mathcal{U} \coloneq \left\{ \bm{u}~|~u_{0} = 0,\: 0 \leq u_{i} \leq s_{i}\:(i = 1,2,\ldots,7) \right\}.
\end{equation}
Since the cardinality of $\mathcal{U}$ is sufficiently small as shown below, the entire set fits in main memory on a typical computer.
\begin{equation}
	\left| \mathcal{U} \right| = \prod_{i = 1}^{7}\left( s_{i} + 1 \right) = 106875.
\end{equation}

By computing the formulas described above, we obtain the cardinality of the entire set $\mathcal{P}_{K}$ of K-canonical positions as follows.
\begin{align}
	|\mathcal{P}_{K}| & = \sum_{\bm{u} \in \mathcal{U}} N(\bm{u}) \notag   \\
	                  & = 808809320797678351777732040093287698 \notag      \\
	                  & \quad \ 12438521503800714936366945233084532 \notag \\
	(                 & \approx 8.09 \times 10^{70}).
\end{align}

Similarly, the total number of K-canonical positions for Mini Shogi is found to be
\begin{equation}
	16014219505238849250 \: \left( \approx 1.6 \times 10^{19} \right).
\end{equation}

Assigning a unique rank in the range from $0$ to $|\mathcal P_K|-1$ to each K-canonical candidate allows uniform sampling by generating uniform random ranks and applying unranking. The canonicality and reachability filters then determine membership in the target set $\mathcal R$, as described in \cref{chapter:estimate-method}.

\subsection{$n$-Ply Predecessor Positions and Reverse Search}\label{chapter:method-prev}
To execute the method described in \cref{chapter:estimate-method}, an algorithm for determining the reachability of each element of the sample $S$ is required. To establish the reachability or unreachability of a position $p$, one approach is to apply a graph search algorithm with the initial position as the start node and $p$ as the goal node. However, if the position $p$ given to the graph search algorithm is unreachable, then without using game-specific pruning, proving unreachability through search alone requires exploring all nodes reachable from the start node. For most games, searching all nodes reachable from the initial position does not terminate within a practical amount of time.

Therefore, we adopted a method of performing a reverse search with $p$ as the start node and the initial position as the goal node. For games where unreachable positions tend to have reverse search trees that terminate within a small number of nodes, it is expected that reachability can be effectively determined by reverse search. As described later, many positions in standard Shogi and Mini Shogi exhibit this property, and consequently, we observed a tendency for computations to terminate within a practical amount of time.

Let $\mathrm{Moves}(p)$ denote the set of all legal moves at position $p$, and let $\mathrm{Next}(p, m)$ denote the position reached when the side to move plays legal move $m$ at position $p$. Then the set $\mathrm{Prev}_n(p)$ of all $n$-ply ($n \ge 1$) predecessor positions of $p$ can be defined as follows.
\begin{align}
	\mathrm{Prev}_{1}(p)     & \coloneq \{ p' \mid \exists m \in \mathrm{Moves}(p') \ \text{s.t.} \ \mathrm{Next} (p', m) = p \} \notag \\
	\mathrm{Prev}_{n + 1}(p) & \coloneq \bigcup_{p' \in \mathrm{Prev}_{n}(p)} \mathrm{Prev}_{1}(p') \quad (n \ge 1).
\end{align}
In Shogi, $\mathrm{Prev}_n(p)$ can be enumerated at speeds comparable to those of forward transitions, enabling fast reverse search.

\section{Detection of Reachable Positions in the Sample}

\subsection{Exclusion of Non-Canonical Positions}\label{chapter:shogi-remove-flip}
For each sampled position $p_j$, we apply a filter that rejects $p_j$ if its horizontal reflection is smaller under the lexicographic order. A rejected position is assigned $X_j=0$ and is not passed to the subsequent filters. It remains included in the denominator $n$ of \eqref{eq:count-estimator}.

For example, in the position shown in \figref{flipH-NG}, the representation of the King placement is $\bm{B}_0 = \left( \emptyset, \emptyset,\{(4, 4)\},\{(4, 8)\} \right)$, which does not change under horizontal mirroring. On the other hand, the representation of the Gold placement $\bm{B}_1 = \left( \emptyset, \emptyset,\{(6, 5), (8, 0)\},\{(1, 7), (3, 0)\} \right)$ becomes $\bm{B}'_1 = \left( \emptyset, \emptyset,\{(0, 0), (2, 5)\},\{(5, 0), (7, 7)\} \right)$ when horizontally mirrored, which is smaller in lexicographic order than the original, so this position is assigned $X_j=0$.

Since the King position restrictions are already imposed during the generation of $\mathcal{P}_K$, the positions excluded here are limited to cases where both the player's King and the opponent's King are on the 5th file.

\begin{figure}[tb]
	\centering
	\begin{minipage}[c]{\columnwidth}
		\centering
		\includegraphics[width=0.85\columnwidth]{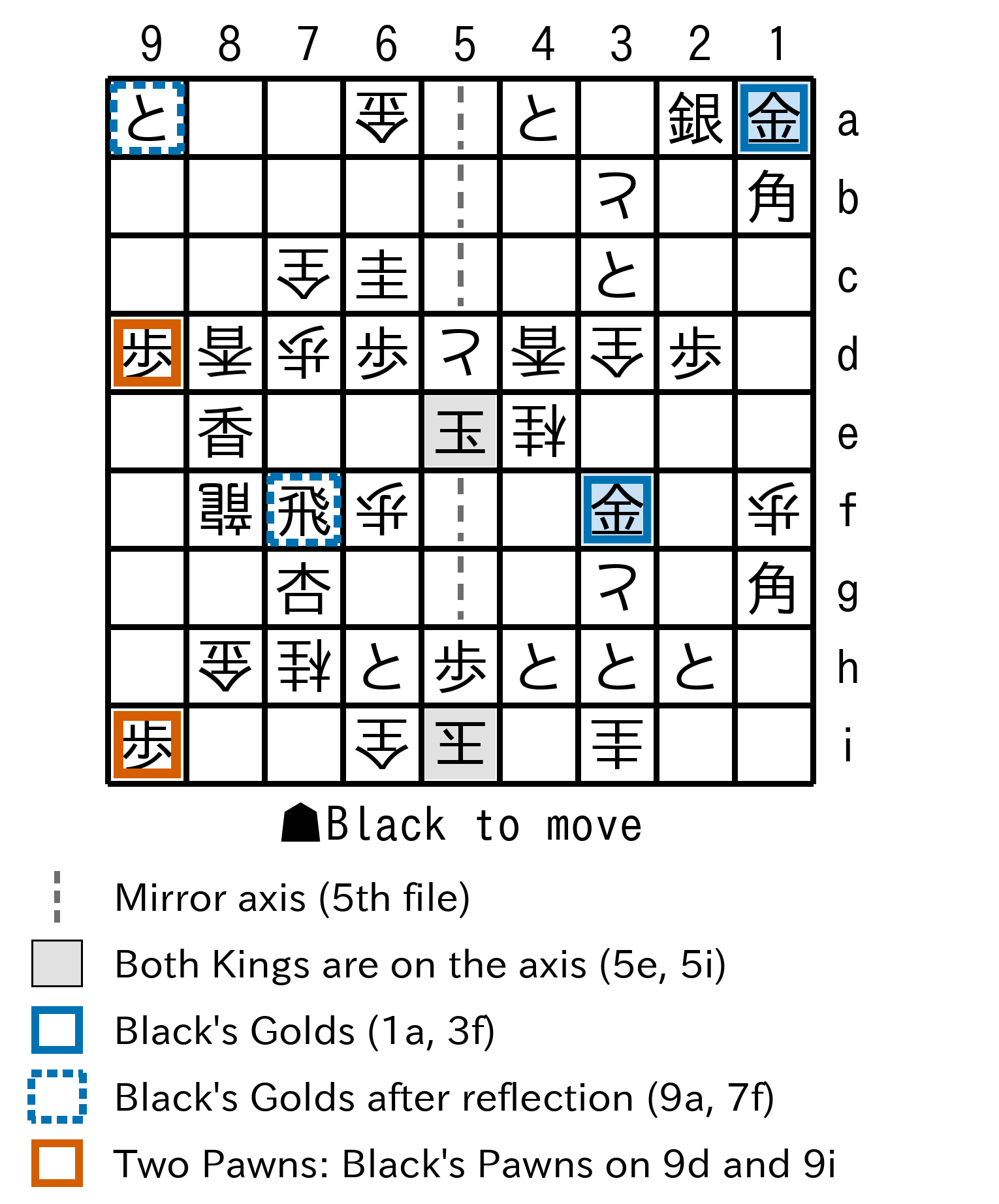}
	\end{minipage}
	\caption{A non-canonical position. Both Kings are on the 5th file, the mirror axis, so the order is determined by the Golds. Reflection moves Black's Golds from 1a and 3f to 9a and 7f, which gives a smaller representation. This position also violates the Two Pawns rule (Black's Pawns on 9d and 9i).}
	\label{flipH-NG}
\end{figure}

\subsection{Exclusion of Unreachable Positions Based on Rule Violations}\label{chapter:shogi-remove-illegal}
Positions violating the following local rules are unreachable. Such sampled positions are assigned $X_j=0$ without reverse search.
\begin{description}[itemsep=0pt]
	\item[Violations related to piece placement] Positions in which the board configuration contains Two Pawns or a piece with no legal moves, such as a Black Knight on the first or second rank. For example, the position in \figref{flipH-NG} contains two unpromoted Black Pawns on the 9th file.
	\item[Violations related to check] Positions in which the King of the non-moving side (the player not to move) is in check. An example is shown in \figref{check-NG}. Such positions are unreachable unless the non-moving side previously made an illegal move (suicide move or ignoring a check).
\end{description}

\begin{figure}[tb]
	\centering
	\begin{minipage}[c]{\columnwidth}
		\centering
		\includegraphics[width=0.85\columnwidth]{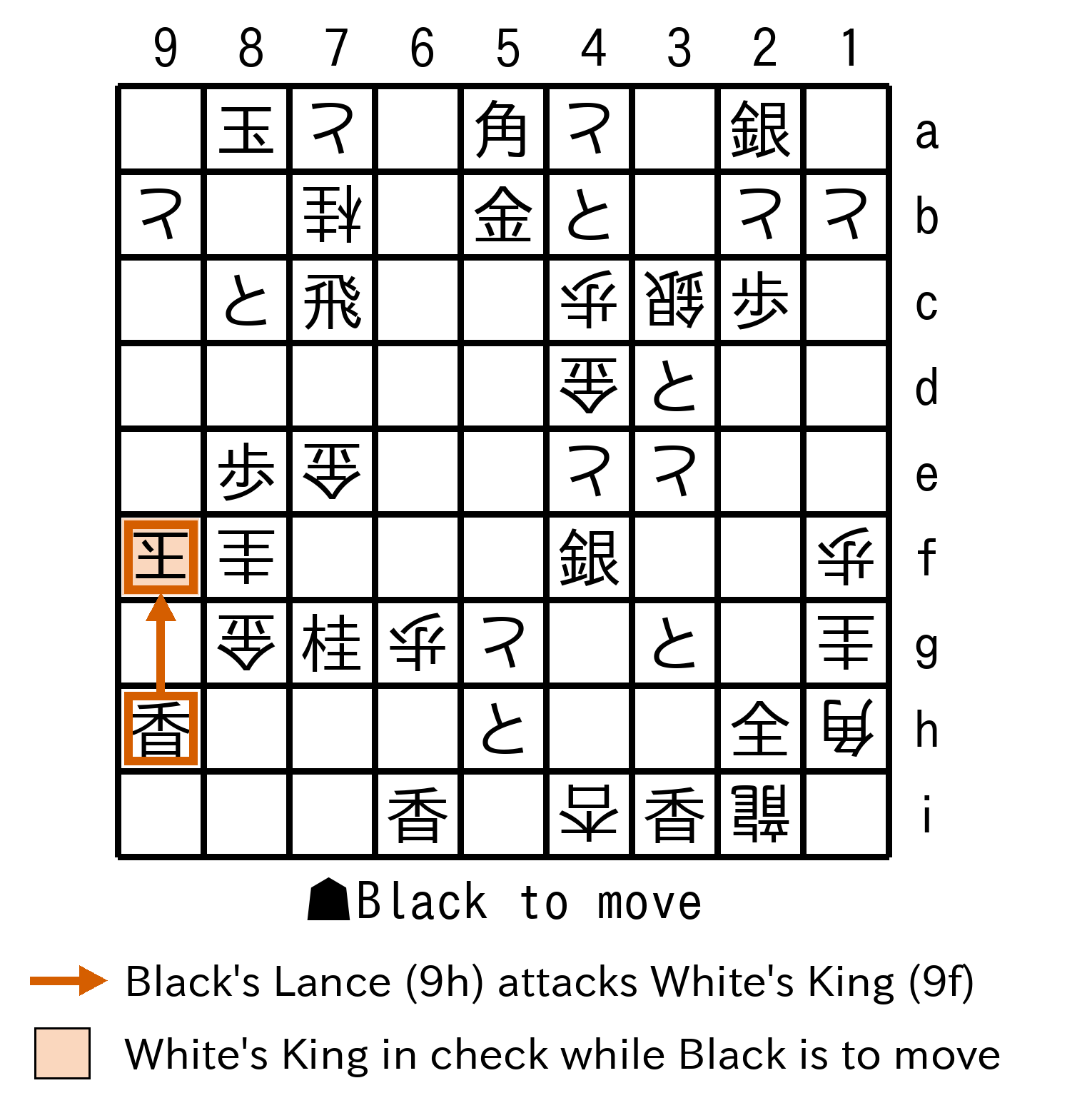}
	\end{minipage}
	\caption{A position violating a local rule. It is Black's turn, but White's King on 9f is in check from Black's Lance on 9h.}
	\label{check-NG}
\end{figure}

\subsection{Determination of Reachable Positions Using Search}\label{chapter:shogi-search-kk}

\subsubsection{KK Positions}
As described in \cref{chapter:method-prev}, a standard reverse search uses $p$ as the start node and the initial position as the goal node. Instead of the initial position, we introduce the set of ``KK (King-King only) positions'' defined below, and propose a method for determining the reachability of Shogi positions by searching from $p$ as the start node toward the set of KK positions as the goal set.

\begin{dfn}[KK position]\label{def:kk}
	A Shogi position $p$ is called a ``KK position'' when it satisfies all of the following conditions.
	\begin{enumerate}[itemsep=0cm]
		\item Only the Black King $K_{\mathrm{B}}$ and the White King $K_{\mathrm{W}}$ are present on the board.
		\item $K_{\mathrm{B}}$ and $K_{\mathrm{W}}$ are separated by a distance of at least two squares in either the $x$ or $y$ direction. That is, denoting the $(x,y)$ coordinates of $K_{\mathrm{B}}$ and $K_{\mathrm{W}}$ as $(x_{K_{\mathrm{B}}}, y_{K_{\mathrm{B}}})$ and $(x_{K_{\mathrm{W}}}, y_{K_{\mathrm{W}}})$, respectively, we have $|x_{K_{\mathrm{B}}} - x_{K_{\mathrm{W}}}| \ge 2 \lor |y_{K_{\mathrm{B}}} - y_{K_{\mathrm{W}}}| \ge 2$.
	\end{enumerate}
\end{dfn}

Among positions satisfying only Condition 1 of \cref{def:kk} but not Condition 2, there are situations where the two Kings give check to each other (and thus involve a suicide move). Condition 2 is imposed to exclude such positions and define ``KK positions'' consistently. Furthermore, from Condition 1, all pieces other than Kings must be held in hand in a KK position. An example of a KK position is shown in \figref{kk-example}. In this example, all pieces except the Kings are held by Black.

\begin{figure}[tb]
	\centering
	\begin{minipage}[c]{\columnwidth}
		\centering
		\includegraphics[width=0.85\columnwidth]{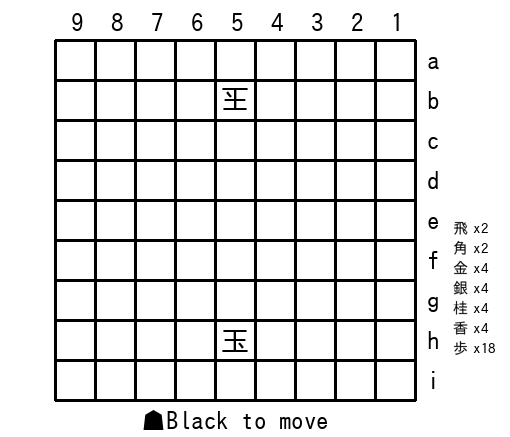}
	\end{minipage}
	\caption{A KK position. Only the two non-adjacent Kings remain on the board; all 38 non-King pieces are in hand. This example assigns all of them to Black, but the KK target set allows any distribution between the two players.}
	\label{kk-example}
\end{figure}

The validity of this method is guaranteed by \cref{thm:kk-reachable} and \cref{thm:kk-unreachable} below. The proof of \cref{thm:initial-kk-each-other} used here is given in \cref{chapter:kk-reachable}.

\begin{thm}\label{thm:kk-reachable}
	If a position $p$ is reachable from a KK position, then $p$ is also reachable from the initial position.
\end{thm}
\begin{proof}
	By \cref{thm:initial-kk-each-other}, every KK position is reachable from the initial position. Therefore, if $p$ is reachable from a KK position, there exists a sequence of moves from the initial position to $p$ via some KK position.
\end{proof}

\begin{thm}\label{thm:kk-unreachable}
	If a position $p$ is unreachable from any KK position, then $p$ is also unreachable from the initial position.
\end{thm}
\begin{proof}
	We prove the contrapositive: ``if $p$ is reachable from the initial position, then $p$ is reachable from a KK position.'' By \cref{thm:initial-kk-each-other}, the initial position is reachable from any KK position. Therefore, if $p$ is reachable from the initial position, there exists a sequence of moves from any KK position to $p$ via the initial position.
\end{proof}

Not every set of positions can replace the initial position as the goal of the reverse search. A goal set must satisfy two conditions: (i) every position in the set is reachable from the initial position, so that reaching the set proves reachability; and (ii) the initial position is reachable from some position in the set, so that every reachable position can reach the set by reverse search. \cref{thm:kk-reachable,thm:kk-unreachable} show that the set of KK positions satisfies both. For example, the set of all positions with few pieces on the board does not satisfy condition (i), since it contains unreachable positions such as those in \cref{chapter:shogi-remove-illegal}. A set of positions close to the initial position, such as those reachable within a few moves, satisfies both conditions, but reaching it still requires returning specific pieces to specific squares, as reaching the initial position does.

Using KK positions as the goal instead of the initial position has two advantages. First, it relaxes the search target from a single position to a large set of positions; in particular, the pieces in hand may be distributed between the two players in any way. Second, whereas reaching the initial position requires returning specific pieces to specific squares, a KK position can be reached simply by removing pieces from the board. In reverse search, undoing a drop returns an unpromoted piece to the hand, and a promoted piece can be removed in the same way after it is unpromoted by undoing a promoting move. This makes it easy to design a heuristic function that guides the search, as described below, and is expected to result in search completion with fewer expanded nodes for reachable positions. For unreachable positions, the search must examine all predecessor positions regardless of the goal set.

The reachability test using KK positions may also be applied to other games with the piece-drop rule, such as other Shogi variants, Crazyhouse Chess, and Bughouse Chess, provided that a goal set satisfying conditions (i) and (ii) is established for the game and the reverse search terminates within a practical amount of time.

\subsubsection{Search from Candidate Positions to KK Positions}
After the filters in \cref{chapter:shogi-remove-illegal}, we determine reachability by reverse search from the candidate position $p$ toward the KK set, using $\mathrm{Prev}_1$ from \cref{chapter:method-prev}. Algorithm~\ref{kk-gbfs} implements Greedy Best-First Search\cite{Wilt-Thayer-Ruml-2010}.

We adopt Greedy Best-First Search because the reachability test only needs to find some reverse path to a KK position, not a shortest one. Search algorithms that take path length into account, such as breadth-first search and A*, spend effort on the optimality of the path, which is irrelevant here. Greedy Best-First Search instead always expands the position in $\mathbf{open}$ with the smallest heuristic value, regardless of the path length so far. The heuristic function only determines the order of expansion and never prunes predecessor positions. Therefore, when $\mathbf{open}$ becomes empty, every position from which $p$ can be reached has been examined.

We define the heuristic function by
\begin{equation}\label{eq:heuristic}
	H(p)=aN(p)+bP(p)+cD(p),
\end{equation}
where $a>0$, $b\geq0$, and $c\geq0$ are real-valued parameters, and
\begin{itemize}[itemsep=0pt]
	\item $N(p)$ is the number of non-King pieces on the board.
	\item $P(p)$ is the number of promoted pieces on the board.
	\item $D(p)$ is the sum, over all promoted pieces on the board, of their vertical distances from their respective owners' promotion zones. A piece inside its promotion zone contributes zero; if no promoted pieces are present, $D(p)=0$.
\end{itemize}

Since $a>0$ and all three terms are nonnegative, $H(p)=0$ holds if and only if $N(p)=0$, that is, only the two Kings are on the board. The input to Algorithm~\ref{kk-gbfs} has passed the filters in \cref{chapter:shogi-remove-illegal}, and $\mathrm{Prev}_1$ never generates a position in which the King of the non-moving side is in check (see the examples in \figref{shogi-no-prev}). Hence, the two Kings are never adjacent in the positions examined by Algorithm~\ref{kk-gbfs}, and for these positions $H(p)=0$ holds if and only if $p$ is a KK position. Algorithm~\ref{kk-gbfs} therefore returns Success exactly when it finds a KK position from which $p$ is reachable. By \cref{thm:kk-reachable,thm:kk-unreachable}, Success means that $p$ is reachable, and Failure means that $p$ is unreachable.

The term $aN(p)$ favors positions with fewer non-King pieces on the board. The term $bP(p)$ penalizes promoted pieces, which cannot be returned to the hand until they are unpromoted. The term $cD(p)$ guides promoted pieces toward the promotion zone, where a promoting move can be undone.

We chose the parameters as follows. In reverse search, a non-King piece leaves the board only when a drop is undone, and a promoted piece becomes unpromoted only when a promoting move is undone. Hence at least $N(p)+P(p)$ reverse moves are required to reach a KK position from $p$, and we set $a=b$ so that the first two terms are proportional to this lower bound. The number of moves needed to bring a promoted piece back to the promotion zone could be estimated for each piece type; instead, we used the vertical distance $D(p)$ as a rough measure, which can exceed the number of moves actually needed, for example for a promoted Rook, which can move several ranks in one move. We therefore gave this term a much smaller weight than the other two. On this basis we set $a=b=10$ and $c=1$ without systematic tuning, so that removing one piece from the board or unpromoting one piece has the same weight as reducing the total distance by ten. We used these values in all experiments in this paper, and every search completed within practical time and memory.

As long as each search runs to completion, the values of $a$, $b$, and $c$ affect only the order of expansion and the computational cost, not the result of the reachability test.

Shinoda\cite{shinoda} presented the position in \figref{shinoda-difficult} as an example where reachability is non-trivial; our algorithm successfully verified its reachability.

\begin{algorithm}[tb]
	\caption{Greedy Best-First Search to KK positions}
	\label{kk-gbfs}
	\begin{algorithmic}[1]
		\AlgInput A position $p$ passing the local rule filters in \cref{chapter:shogi-remove-illegal}.
		\AlgOutput Success if $p$ is reachable; Failure otherwise.
		\Function {can\_reach\_KK}{$p$}
		\State $\mathbf{open} \gets \{p\}$
		\State $\mathbf{closed} \gets \{\}$
		\While{$\mathbf{open} \neq \emptyset$}
		\State $p^* \gets$ $\underset{p \ \in \ \mathbf{open}} {\operatorname{argmin}} \ H(p)$
		\State $\mathbf{open} \gets \mathbf{open} - \{p^*\}$
		\State $\mathbf{closed} \gets \mathbf{closed} \ \cup \ \{p^*\}$
		\If{$H(p^*) = 0$}
		\Return Success
		\EndIf
		\ForAll{$p' \in \mathrm{Prev}_1(p^*)$}
		\If{$p' \not \in (\mathbf{closed} \ \cup \ \mathbf{open})$}
		\State $\mathbf{open} \gets \mathbf{open} \ \cup \ \{p'\}$
		\EndIf
		\EndFor
		\EndWhile
		\Return Failure
		\EndFunction
	\end{algorithmic}
\end{algorithm}

\begin{figure}[tb]
	\centering
	\begin{minipage}[c]{\columnwidth}
		\centering
		\includegraphics[width=0.85\columnwidth]{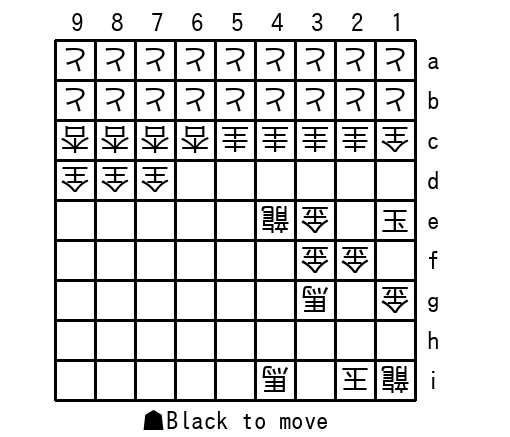}
	\end{minipage}
	\caption{A reachable position that is difficult to prove}
	\label{shinoda-difficult}
\end{figure}

\begin{figure*}[tb]
	\centering
	\includegraphics[width=2\columnwidth]{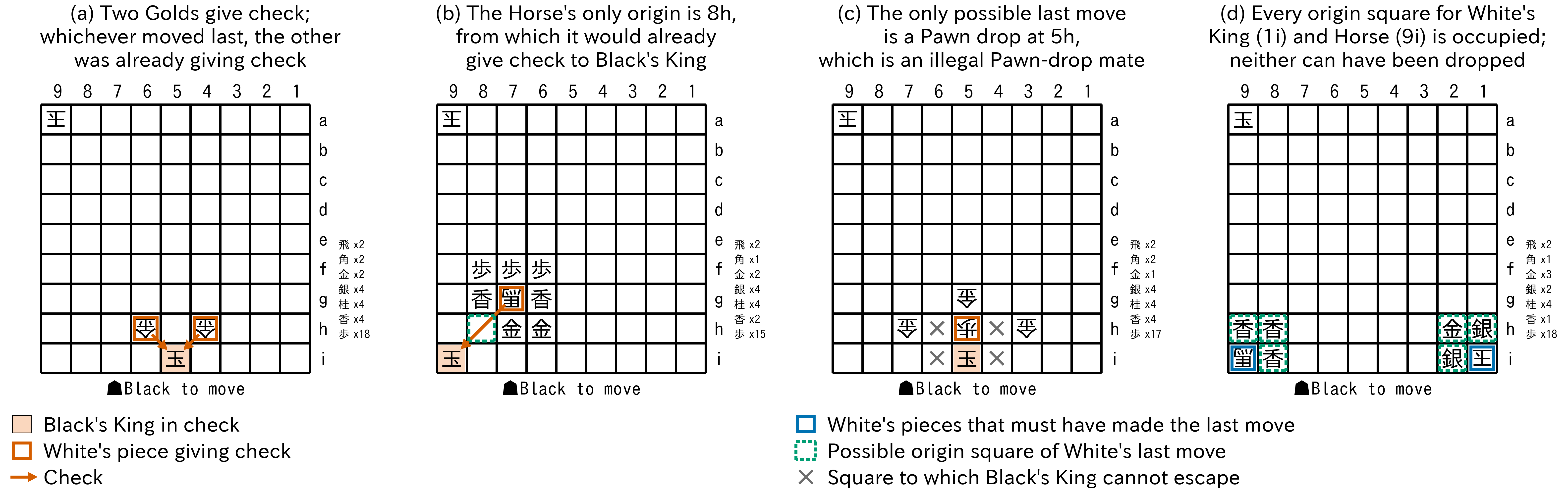}
	\caption{Positions with no legal one-ply predecessor, all with Black to move. In (a) and (b), every possible reconstruction of White's last move leaves Black's King already in check before that move. In (c), the only possible last move is an illegal Pawn-drop mate at 5h. In (d), neither of White's pieces on the board (the King on 1i and the Horse on 9i) has a possible origin square, and neither could have been dropped.}
	\label{shogi-no-prev}
\end{figure*}

Examples of positions $p$ satisfying $\mathrm{Prev}_1(p) = \emptyset$, i.e., positions with no 1-ply predecessor, are shown in \figref{shogi-no-prev}. The details are as follows.

\begin{itemize}[itemsep=0pt]
	\item Positions where Black's King is in check and going back one move would result in a position in which the non-moving side's King is already in check (\figref{shogi-no-prev}(a) and (b)).
	      \begin{itemize}
		      \item In \figref{shogi-no-prev}(a), the White Golds on 4h and 6h are giving checks to Black's King. Regardless of what White's previous move was, the position one move prior would be a position in which the non-moving side's King is already in check.
		      \item In \figref{shogi-no-prev}(b), the piece that made the last move is uniquely determined to be White's Horse on 7g (it cannot be a discovered check because there is no piece that could have moved from 8h). Whether the preceding move was B-8h $\rightarrow$ +B-7g or +B-8h $\rightarrow$ +B-7g, the position one move prior would be a position in which the non-moving side's King is already in check.
	      \end{itemize}
	\item Positions where no 1-ply predecessor exists due to the Drop Pawn Mate prohibition (\figref{shogi-no-prev}(c)). This position is a checkmate position. The preceding move is uniquely determined to be dropping a Pawn from hand onto 5h, but this move constitutes Drop Pawn Mate (an illegal move). Hence, the position has no 1-ply predecessor.
	\item Positions where no possible origin square exists for any of the opponent's pieces on the board (\figref{shogi-no-prev}(d)). Since this is a Black-to-move position, one of White's pieces on the board (the King on 1i, the Horse on 9i) must have moved on the previous turn. Note that the move could not have been a drop, because all of White's pieces on the board are either the King or promoted pieces, and passing is not a legal move in Shogi. However, neither piece has a vacant square that could serve as its origin.
\end{itemize}

\begin{figure*}[tb]
	\centering
	\includegraphics[width=1.4\columnwidth]{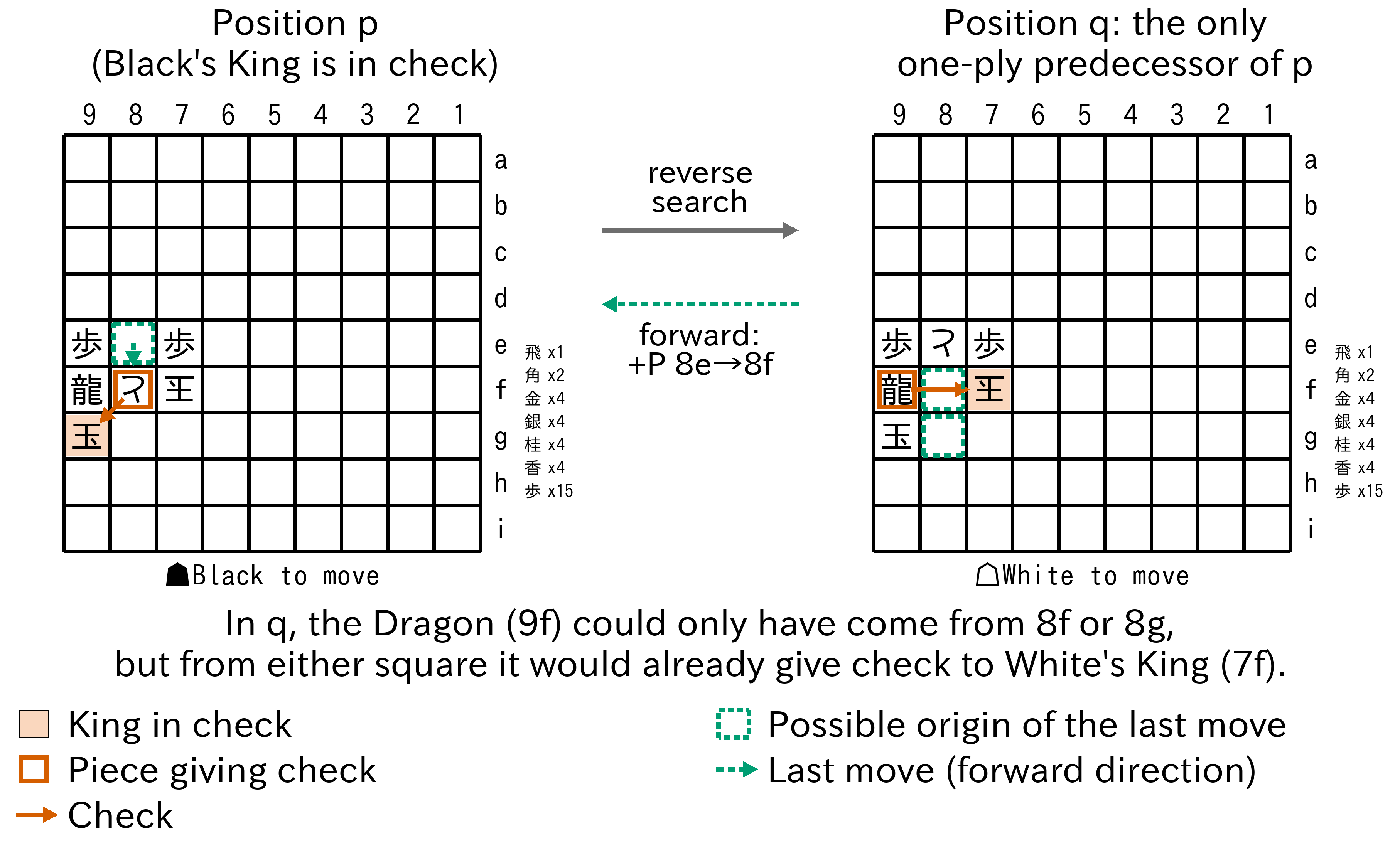}
	\caption{Position $p$ (left) has position $q$ (right) as its only one-ply predecessor, obtained by undoing White's Tokin move from 8e to 8f. In $q$, the Dragon on 9f could not legally have moved from either 8f or 8g. Thus, $\mathrm{Prev}_1(p)=\{q\}$ and $\mathrm{Prev}_1(q)=\emptyset$.}
	\label{shogi-no-prev2}
\end{figure*}

Positions $p$ satisfying $\mathrm{Prev}_1(p)\neq\emptyset$ and $\mathrm{Prev}_2(p)=\emptyset$ also exist. Consider positions $p$ and $q$ in \figref{shogi-no-prev2}. In $p$, the Tokin on 8f checks Black's King, and undoing its move from 8e to 8f gives the only possible one-ply predecessor $q$. In $q$, the Dragon on 9f checks White's King. Its only possible origin squares are 8f and 8g, but from either square it would already check White's King before Black's move. Therefore, $q$ has no legal one-ply predecessor, and $p$ has no legal two-ply predecessor. Although this example was constructed manually, many positions with $\mathrm{Prev}_{i}(p)\neq\emptyset$ and $\mathrm{Prev}_{i+1}(p)=\emptyset$ found in our experiments involved a similar obstruction caused by check.

\section{Experiments}
We first conducted experiments on Mini Shogi to verify whether the proposed method works correctly. We then extended the experimental code used for Mini Shogi to conduct experiments on standard Shogi.

\subsection{Mini Shogi}
We conducted experiments with the sample size $|S|$ set to 100 million. We implemented experimental code in Python, prioritizing code readability and development efficiency\footnote{\url{https://github.com/u-tokyo-gps-tanaka-lab/minishogi-position-ranking}}, and the computation was performed on a machine with 128 GB of RAM and an AMD Ryzen Threadripper 2990WX, using 64 parallel processes for approximately 3 days. The parameters of the heuristic function were set to $a=10, b=10, c=1$. In the reachability determination by searching toward KK positions, no limits were imposed on the number of nodes in $\mathbf{open}$ or $\mathbf{closed}$. For the given positions, no memory overflow occurred, and we obtained success/failure determinations within a practical amount of time. As a result, $14{,}849{,}198$ sampled positions were found to be canonical and reachable. The number of positions passing each filter is shown in \tabref{shogi-result-tbl}.

The sample proportion is $\hat p = 0.14849198$. Using the normal approximation to the binomial distribution, the approximate $3\sigma$ confidence interval for the proportion is $0.14849 \pm 0.00011$. Multiplying by $|\mathcal P_K|$ gives an estimated number of positions of $(2.3780 \pm 0.0017) \times 10^{18}$, which is $2.38\times10^{18}$ to three significant digits. Examples of positions determined to be reachable are shown in \figref{minishogi-reach-ok}.

\begin{figure}[tb]
	\centering
	\begin{minipage}[c]{\columnwidth}
		\centering
		\includegraphics[width=1\columnwidth]{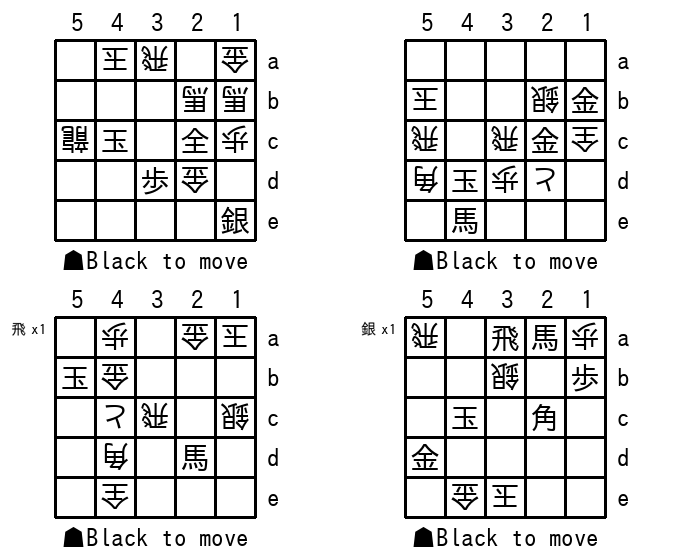}
	\end{minipage}
	\caption{Mini Shogi positions determined to be reachable}
	\label{minishogi-reach-ok}
\end{figure}

\subsection{Shogi}\label{chapter:exp-shogi}
We conducted experiments with the sample size $|S|$ set to 5 billion. In this experiment, the following two programs were used in combination.

\begin{itemize}[itemsep=0pt]
	\item A Python program developed from scratch\footnote{\url{https://github.com/u-tokyo-gps-tanaka-lab/shogilib/}}. This implements all the filters described in \cref{chapter:shogi-remove-flip}, \cref{chapter:shogi-remove-illegal}, and \cref{chapter:shogi-search-kk}.
	\item A program based on the open-source C++ Shogi engine ``YaneuraOu'' \cite{yaneuraou}\footnote{\url{https://github.com/u-tokyo-gps-tanaka-lab/YaneuraOu/tree/estimate_number_of_positions}}. This implements the ``violations related to check'' from \cref{chapter:shogi-remove-illegal} and the reachability determination described in \cref{chapter:shogi-search-kk}.
\end{itemize}

The canonicalization described in \cref{chapter:shogi-remove-flip} and the ``violations related to piece placement'' in \cref{chapter:shogi-remove-illegal} were performed using the Python program, while the ``violations related to check'' in \cref{chapter:shogi-remove-illegal} and the reachability test in \cref{chapter:shogi-search-kk} were performed using the modified YaneuraOu to reduce computation time. The parameters of the heuristic function were set to $a=10, b=10, c=1$. The computation was performed on a machine with 128 GB of RAM and an AMD Ryzen Threadripper 2990WX. It required approximately 131 hours with 56 parallel processes. Additionally, we conducted a separate experiment with the sample size reduced to 500 million, performing all processes in \cref{chapter:shogi-remove-flip}, \cref{chapter:shogi-remove-illegal}, and \cref{chapter:shogi-search-kk} using only the Python program. The results matched those obtained with the modified YaneuraOu, confirming the correctness of both implementations.

In the search toward KK positions, no limits were imposed on the number of nodes in $\mathbf{open}$ or $\mathbf{closed}$. For the given positions, no memory overflow occurred, and success/failure determination results were obtained within a practical amount of time. As a result, out of 5 billion samples, 40,491,613 sampled positions were found to be canonical and reachable. The number of positions passing each filter is shown in \tabref{shogi-result-tbl}.

\begin{table}[tb]
	\caption{Number of samples passing each filter}
	\label{shogi-result-tbl}
	\hbox to\hsize{\hfil
		\begin{tabular}{c||r|r}
			\hline
			Filter               & Shogi         & Mini Shogi  \\
			\hline
			\hline
			Initial generation   & 5,000,000,000 & 100,000,000 \\
			Horizontal mirroring & 4,945,063,843 & 96,774,076  \\
			Pawn placement       & 187,220,063   & 77,795,825  \\
			Opponent King check  & 58,981,117    & 21,506,911  \\
			Reachability         & 40,491,613    & 14,849,198  \\
			\hline
		\end{tabular}
		\hfil}
\end{table}

From this result, the sample proportion of candidates that are both canonical and reachable is
\begin{equation}
	\hat{p} = \frac{40{,}491{,}613}{5 \times 10^9} \approx 8.0983 \times 10^{-3}.
\end{equation}
Using the normal approximation to the binomial distribution, the approximate $3\sigma$ confidence interval for the proportion is $(8.0983 \pm 0.0038) \times 10^{-3}$. Multiplying by $|\mathcal P_K|$ gives an estimated number of positions of $(6.5500 \pm 0.0031) \times 10^{68}$. Both ends of this interval round to $6.55\times10^{68}$ at three significant digits.

We chose the sample size so that the estimate would be determined to three significant digits. Under the normal approximation, the relative half-width of the $3\sigma$ interval is approximately $3\sqrt{(1-p_{\mathcal R})/(np_{\mathcal R})}$. A preliminary experiment with one million samples yielded 8,083 canonical and reachable positions, giving $\hat p_0=0.008083$ and an estimate of approximately $6.54\times10^{68}$. For the half-width to be at most $0.005\times10^{68}$, half a unit in the third significant digit, approximately $1.9\times10^9$ samples are required; with 500 million samples, the half-width would be about $0.010\times10^{68}$. We therefore set the sample size to five billion, about 2.6 times this requirement.

Examples of positions determined to be reachable are shown in \figref{shogi-reach-ok}.

\begin{figure}[tb]
	\centering
	\begin{minipage}[c]{\columnwidth}
		\centering
		\includegraphics[width=1\columnwidth]{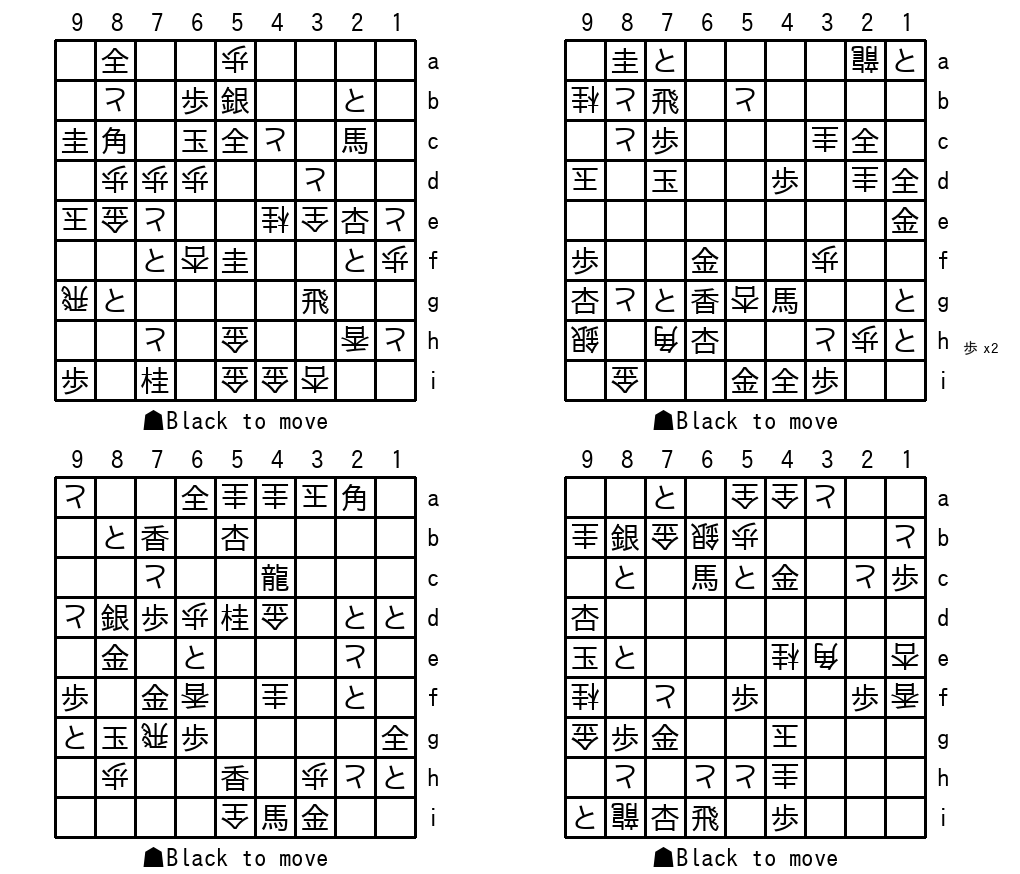}
	\end{minipage}
	\caption{Shogi positions determined to be reachable}
	\label{shogi-reach-ok}
\end{figure}

Note that the results of this experiment also allow us to compute the state-space complexity under the position definition used in previous studies\cite{shinoda, miyako} without any additional experiments. It is approximately $1.31 \times 10^{69}$.

\subsection{Discussion of Experimental Results}
For positions determined to be unreachable by the search toward KK positions, the distribution of the maximum number of plies that could be traced back through predecessor positions is shown in \tabref{count-cant-reach}. The position traced back the furthest (8 plies) during the Mini Shogi experiment is shown in \figref{minishogi-long-noreach}. Furthermore, one example of a position that could be traced back the furthest (3 plies), discovered in the Shogi experiment, is shown in \figref{shogi-long-noreach}.

\begin{figure}[tb]
	\centering
	\begin{minipage}[c]{\columnwidth}
		\centering
		\includegraphics[width=1\columnwidth]{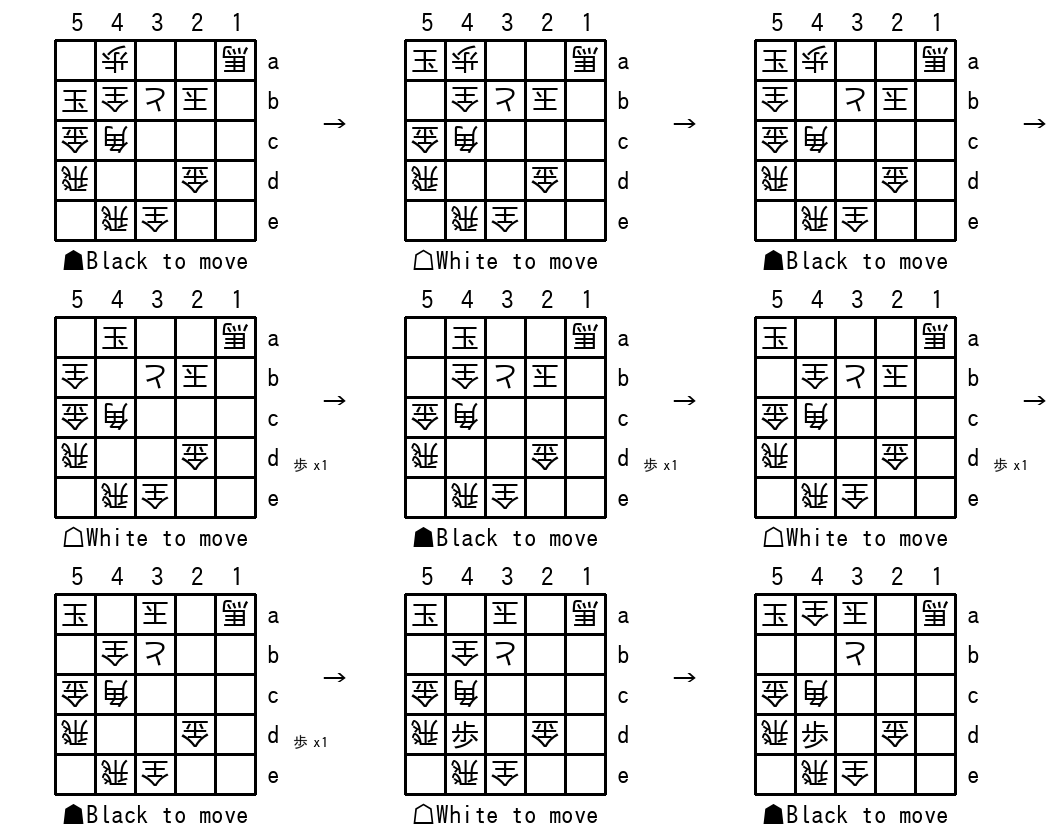}
	\end{minipage}
	\caption{A Mini Shogi position where 8 plies could be traced back. Arrows indicate the move sequence in the actual game.}
	\label{minishogi-long-noreach}
\end{figure}

\begin{figure}[tb]
	\centering
	\begin{minipage}[c]{\columnwidth}
		\centering
		\includegraphics[width=1\columnwidth]{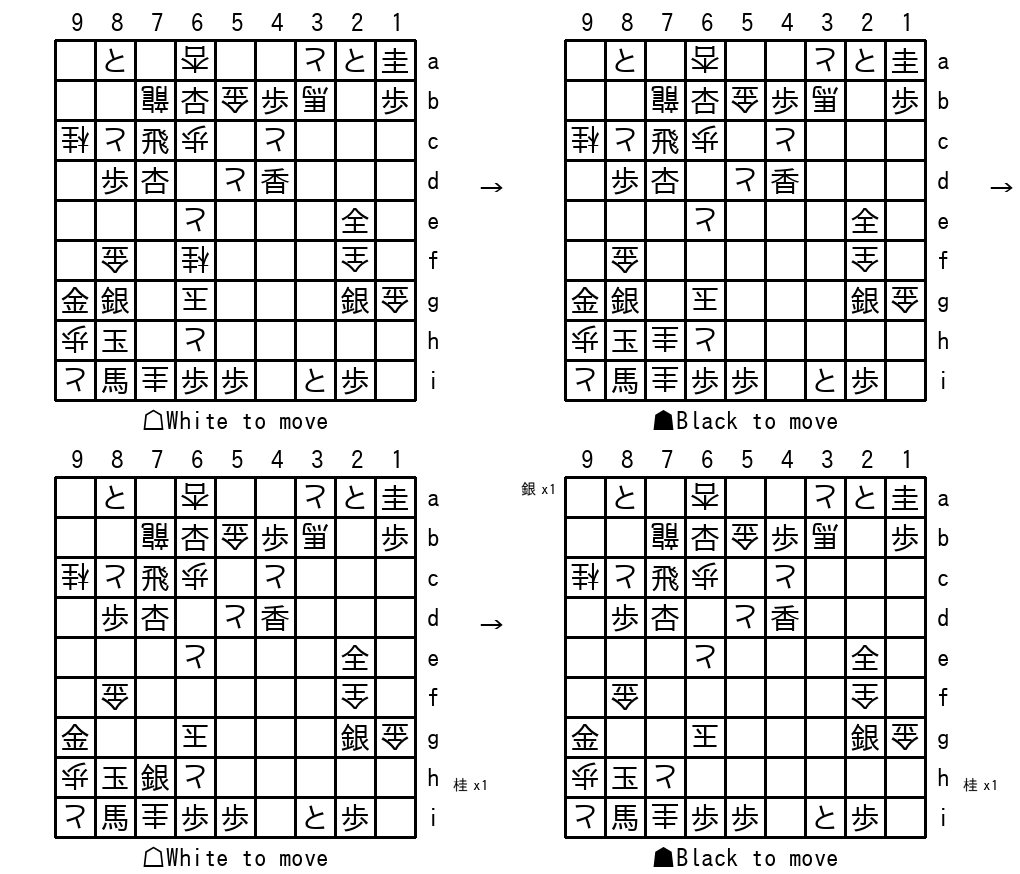}
	\end{minipage}
	\caption{A Shogi position where 3 plies could be traced back. Arrows indicate the move sequence in the actual game.}
	\label{shogi-long-noreach}
\end{figure}

\begin{table}[tb]
	\caption{Number of positions determined unreachable by reverse search and the number of plies traced back}
	\label{count-cant-reach}
	\hbox to\hsize{\hfil
		\begin{tabular}{c||r|r}
			\hline
			Max plies traced back & Shogi      & Mini Shogi \\
			\hline
			0                     & 18,488,938 & 6,650,818  \\
			1                     & 480        & 4,175      \\
			2                     & 83         & 2,494      \\
			3                     & 3          & 114        \\
			4                     & 0          & 104        \\
			5                     & 0          & 2          \\
			6                     & 0          & 4          \\
			7                     & 0          & 1          \\
			8                     & 0          & 1          \\
			9                     & 0          & 0          \\
			\hline
		\end{tabular}
		\hfil}
\end{table}

We conjecture that unreachable positions traceable back further than those discovered here may exist; however, determining the maximum traceback depth for such positions remains an open problem. In our experiments, the reverse search from every unreachable position reached a dead end within 3 plies for Shogi and 8 plies for Mini Shogi. This property, that the search from unreachable positions does not diverge, is important for the viability of our method.

For 10,000 positions randomly sampled from the reachable positions found in the Shogi experiment, \figref{shogi-h-scatter} shows the length $L$ of the path to a KK position found by Algorithm~\ref{kk-gbfs}, measured as the number of positions on the path including both ends, and the number $E$ of expanded nodes, that is, the number of positions removed from $\mathbf{open}$, including the KK position found. \figref{minishogi-h-scatter} shows the same quantities for 10,000 positions randomly sampled from the reachable positions found in the Mini Shogi experiment, with the same heuristic function and parameters. The red line indicates $E=L$. Since every position on the path found is expanded, no point lies below this line, and the vertical distance from the line corresponds to the number of expansions outside the path.

In Shogi, the median of $E/L$ was 1.07, its 95th percentile was 1.43, and $E$ was less than three times $L$ for all sampled positions. In Mini Shogi, the median of $E/L$ was 1.06 and $E=L$ held for 43.5\% of the sampled positions, compared with 17.8\% in Shogi. The distribution has a longer tail, however: the 95th percentile of $E/L$ was 2.28, $E \geq 3L$ held for 94 of the 10,000 positions, and the largest ratio was $E/L=14.3$ ($L=24$, $E=344$). These results suggest that, in both games, the reverse search from most reachable positions proceeds toward KK positions with little backtracking, although in Mini Shogi a small fraction of positions requires considerably more expansions than the path length. Note that the path found is not necessarily the shortest.

\begin{figure}[tb]
	\centering
	\begin{minipage}[c]{\columnwidth}
		\centering
		\includegraphics[width=0.85\columnwidth]{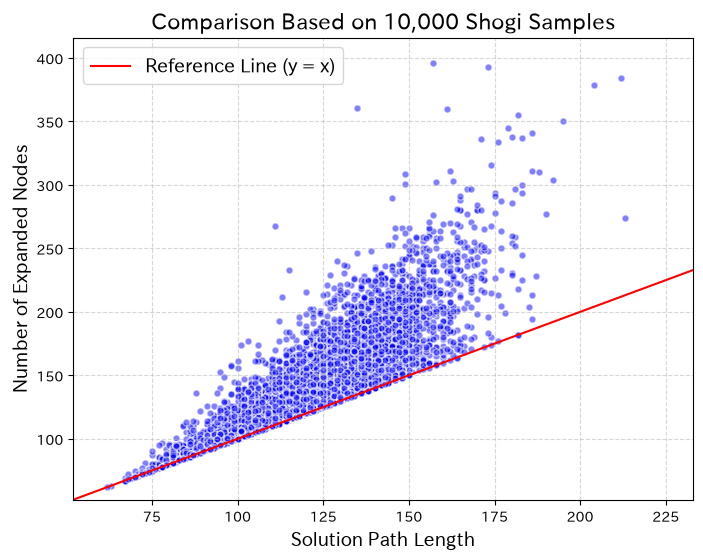}
	\end{minipage}
	\caption{Path length $L$ and number of expanded nodes $E$ for 10,000 reachable Shogi positions. The red line is the reference line $E=L$.}
	\label{shogi-h-scatter}
\end{figure}

\begin{figure}[tb]
	\centering
	\begin{minipage}[c]{\columnwidth}
		\centering
		\includegraphics[width=0.85\columnwidth]{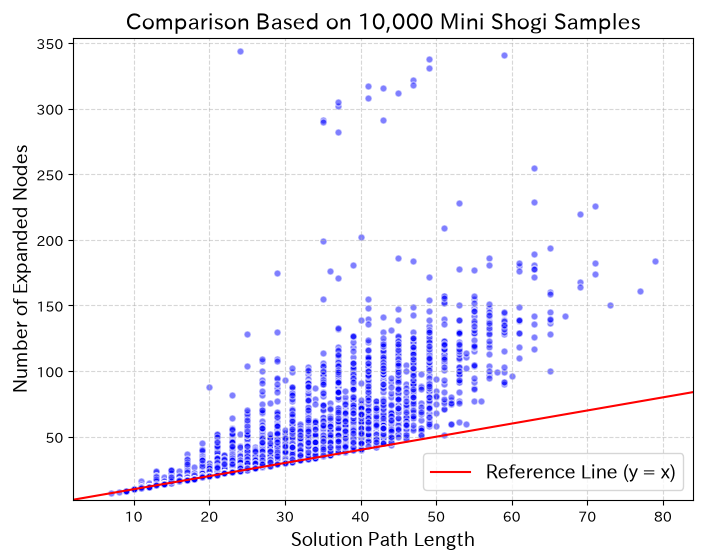}
	\end{minipage}
	\caption{Path length $L$ and number of expanded nodes $E$ for 10,000 reachable Mini Shogi positions. The red line is the reference line $E=L$.}
	\label{minishogi-h-scatter}
\end{figure}

Note that among the generated positions, not a single position was found where $\mathrm{Prev}_1(p) = \emptyset$ due to the Drop Pawn Mate prohibition or where stalemate occurred. This suggests that the choice of whether to treat Drop Pawn Stalemate as Drop Pawn Mate does not significantly affect the estimate.

\section{Conclusion}
We estimated the number of canonical reachable positions in Shogi, where a position consists of the side to move, the board configuration, and the pieces in hand, without move-history information. Positions are normalized to Black to move, and horizontal reflections are identified. Uniform sampling from an exactly enumerated candidate set, followed by canonicality and reachability tests, yielded an estimate of $6.55 \times 10^{68}$ for Shogi and $2.38 \times 10^{18}$ for Mini Shogi, both to three significant digits. The proposed method can be applied to any game for which one can define a set of positions mutually reachable with the initial position (corresponding to KK positions) and an appropriate heuristic function $H$, provided that the reverse search terminates within practical time and memory bounds. For future work, it would be interesting to investigate how the number of positions changes in other Shogi variants or board games.

\begin{acknowledgment}
	This work was supported by JSPS KAKENHI Grant Number JP24K15244.
\end{acknowledgment}

\bibliographystyle{unsrt}
\bibliography{en}

\appendix

\section{Mutual Reachability Between the Initial Position and KK Positions}\label{chapter:kk-reachable}
In this section, we prove that any KK position and the initial position are mutually reachable. Board coordinates follow \figref{board-coordinates}, and $h$ denotes Black's pieces in hand. The proof proceeds in three stages. First, we show that the side to move and the King locations can be changed within the KK set while keeping the hands fixed. Second, we show that pieces can be transferred between the two players' hands, establishing mutual reachability of all KK positions. Finally, we connect one KK position to the initial position in each direction.

\begin{dfn}
	We define the following sets of KK positions.
	\begin{itemize}[itemsep=0pt]
		\item $BT_{k, h}$ : The set of KK positions with Black's pieces in hand $h$ such that $1 \leq k \leq 7$, $y_{K_\mathrm{W}} < k$ and $y_{K_\mathrm{B}} > k$.
		\item $TB_{k, h}$ : The set of KK positions with Black's pieces in hand $h$ such that $1 \leq k \leq 7$, $y_{K_\mathrm{B}} < k$ and $y_{K_\mathrm{W}} > k$.
		\item $LR_{k, h}$ : The set of KK positions with Black's pieces in hand $h$ such that $1 \leq k \leq 7$, $x_{K_\mathrm{B}} < k$ and $x_{K_\mathrm{W}} > k$.
		\item $RL_{k, h}$ : The set of KK positions with Black's pieces in hand $h$ such that $1 \leq k \leq 7$, $x_{K_\mathrm{W}} < k$ and $x_{K_\mathrm{B}} > k$.
	\end{itemize}
\end{dfn}

By the definition of KK positions, the Black King and White King are separated by at least 2 in either the $x$ or $y$ coordinate. Therefore, every KK position belongs to at least one of the above sets (some positions belong to two or more sets simultaneously).

In the above four sets, we call the region in which the Kings can move while satisfying the set membership conditions the ``movable region.'' For example, in the set $BT_{k, h}$, the movable region for Black's King is defined by $0 \leq x \leq 8, k < y \leq 8$, and the one for White's King is defined by $0 \leq x \leq 8, 0 \leq y < k$.

As long as the Kings move within the movable region, neither ignoring a check nor a suicide move can occur. In the following, we call the extent of the movable region in the $y$ direction (the $x$ direction for $LR_{k, h}$ and $RL_{k, h}$) the ``depth.''

\begin{lemma}
	From any KK position $(t, \bm{b}, \bm{h})$, the KK position $(\neg t, \bm{b}, \bm{h})$ with the same board and pieces in hand but with the side to move swapped is reachable.
\end{lemma}
\begin{proof}
	First, consider the case where the KK position $(t, \bm{b}, \bm{h})$ has $t = \mathrm{B}$ (Black to move) and belongs to $BT_{k, h}$. We construct explicit move sequences for two cases.
	\begin{itemize}
		\item $1 \leq k \leq 6$: The depth of Black's movable region is at least 2. Therefore, we can construct a move sequence (hereafter $C_3$) in which Black returns to the original square in 3 moves within the movable region, as shown in the left of \figref{cycle_2_3}. Similarly, we can construct a move sequence (hereafter $C_2$) in which White returns to the original square in 2 moves. By having Black play $C_3$ and White play $C_2$, the side to move switches to White.
		\item $k = 7$: The depth of White's movable region is at least 2. Therefore, we can construct a $C_3$ move sequence for White, as shown in the right of \figref{cycle_2_3}. By concatenating one $C_3$ for White and two $C_2$'s for Black, the side to move switches to White.
	\end{itemize}
	By symmetry, the side to move can be swapped similarly for $t = \mathrm{W}$ and for $TB_{k, h}$, $LR_{k, h}$, and $RL_{k, h}$.
\end{proof}

\begin{figure}[tb]
	\centering
	\begin{minipage}[c]{\columnwidth}
		\centering
		\includegraphics[width=1\columnwidth]{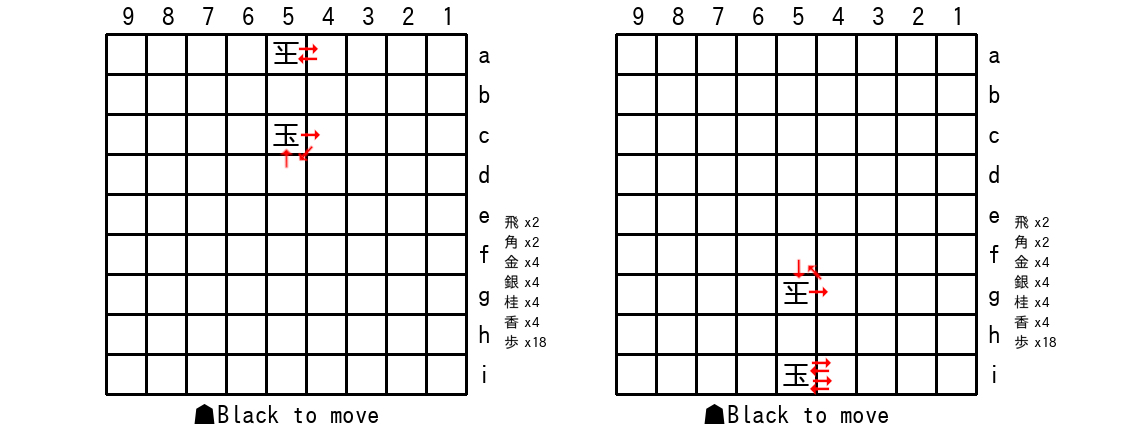}
	\end{minipage}
	\caption{Move sequences for changing the side to move}
	\label{cycle_2_3}
\end{figure}

\begin{figure*}[tb]
	\centering
	\includegraphics[width=2\columnwidth]{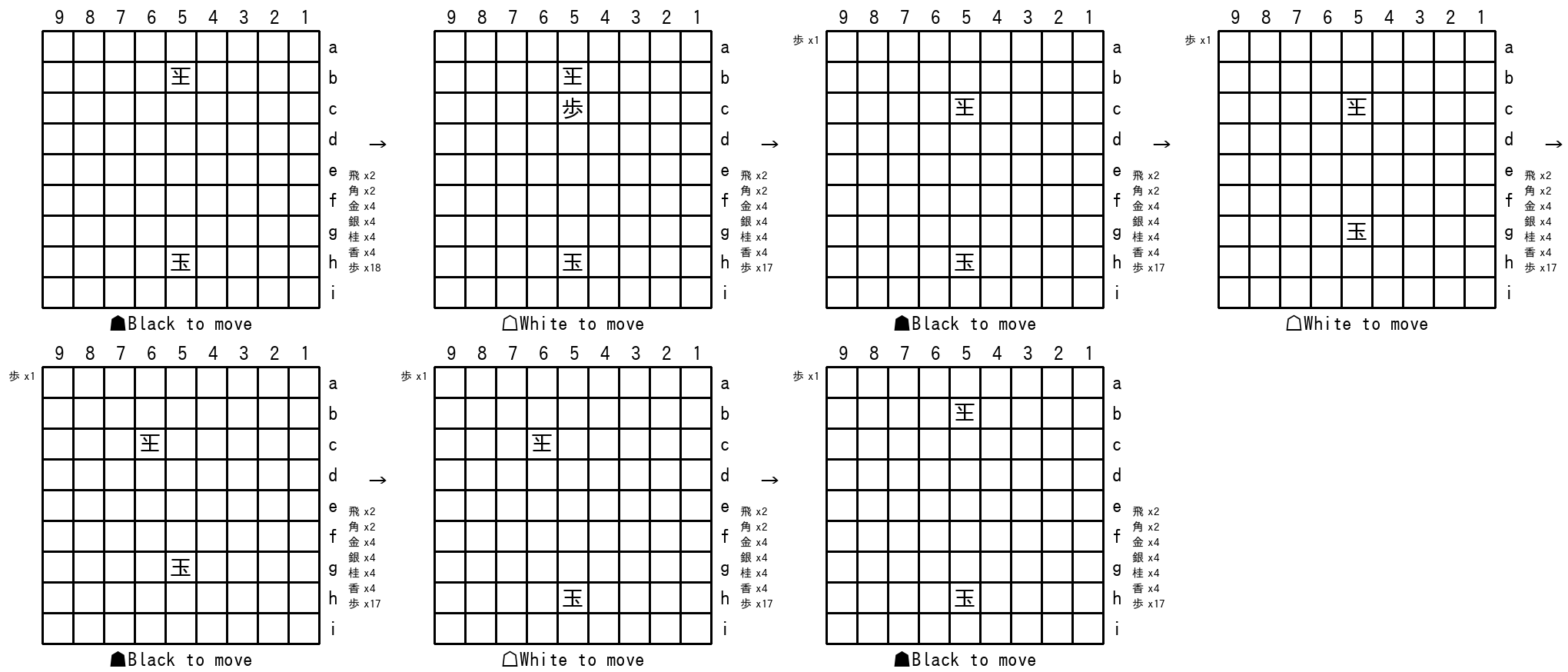}
	\caption{An example of transferring a piece in hand from Black to White in a KK position}
	\label{pawn_to_player2}
\end{figure*}

\begin{lemma}\label{lemma:in_area}
	For any two distinct KK positions $( t, (p_1 , p_2), \bm{h})$, $( t', (p'_1 , p'_2), \bm{h}) \in BT_{k, h}$, $( t, (p_1 , p_2), \bm{h})$ can reach $( t', (p'_1 , p'_2), \bm{h} )$.
\end{lemma}
\begin{proof}
	We show the reachability from $(\mathrm{B}, (p_1 , p_2), \bm{h})$ to $(\mathrm{B}, (p'_1 , p'_2), \bm{h})$. Once this is shown, by the previous lemma, the reachability from $( t, (p_1 , p_2), \bm{h} )$ to $( t', (p'_1 , p'_2), \bm{h} )$ also follows.

	Let $d ( c , c' )$ denote the number of moves required to move from square $c$ to square $c'$ along the shortest path within the movable region. Then the following holds.
	\begin{equation}
		d (c,c') = \max(|c_x - c'_x|, |c_y - c'_y|).
	\end{equation}
	If the two shortest-path lengths $d(p_1,p'_1)$ and $d(p_2,p'_2)$ have the same parity, append $C_2$ move sequences to the shorter path to make their lengths equal. Interleaving the two paths, starting with Black, reaches the target King placement with Black to move. If the lengths have opposite parity, append $C_2$ move sequences so that Black's path is exactly one move longer than White's. Interleaving these paths reaches the target King placement with White to move. The preceding lemma then changes the side to move to Black without changing the board or the hands.
\end{proof}

\begin{lemma}\label{lemma:between_area1}
	For any $k$ and $k'$, $s \in BT_{k, h}$ and $s' \in LR_{k', h}$ are mutually reachable. Similarly, mutual reachability holds between $BT_{k, h}$ and $RL_{k', h}$, between $TB_{k, h}$ and $LR_{k', h}$, and between $TB_{k, h}$ and $RL_{k', h}$.
\end{lemma}
\begin{proof}
	For any $1 \leq k \leq 7$ and $1 \leq k' \leq 7$, $BT_{k, h}$ and $LR_{k', h}$ contain a KK position that belongs to both: $(\mathrm{B}, ((k' - 1 , k + 1) , (k' + 1 , k - 1)), \bm{h})$. Therefore, by going through this position, any $BT_{k, h}$ position and any $LR_{k', h}$ position are mutually reachable.
\end{proof}

\begin{lemma}\label{lemma:between_area2}
	KK positions $(t, \bm{b}, \bm{h})$ and $(t', \bm{b'}, \bm{h})$ with the same pieces in hand are mutually reachable.
\end{lemma}
\begin{proof}
	It suffices to show mutual reachability for the following pairs.
	\begin{itemize}[itemsep=0pt]
		\item A $BT_{k, h}$ position and a $TB_{k', h}$ position
		\item A $BT_{k, h}$ position and a $BT_{k', h}$ position
		\item A $TB_{k, h}$ position and a $TB_{k', h}$ position
		\item An $LR_{k, h}$ position and an $RL_{k', h}$ position
		\item An $LR_{k, h}$ position and an $LR_{k', h}$ position
		\item An $RL_{k, h}$ position and an $RL_{k', h}$ position
	\end{itemize}
	To show mutual reachability between any $BT_{k, h}$ position and $TB_{k', h}$ position, it suffices to note that for some $1 \leq k'' \leq 7$, any $BT_{k, h}$ position and $LR_{k'', h}$ position are mutually reachable, and any $TB_{k', h}$ position and $LR_{k'', h}$ position are mutually reachable. The other pairs are similarly mutually reachable, establishing that all KK positions with the same pieces in hand are mutually reachable.
\end{proof}

\begin{lemma}\label{lemma:kk5258}
	All KK positions with Black's King on 5h and White's King on 5b are mutually reachable.
\end{lemma}
\begin{proof}
	In a KK position where Black's King is on 5h, White's King is on 5b, and it is Black's turn, any type of piece held by Black can be transferred to White's hand. In fact, by the procedure shown in \figref{pawn_to_player2}, we can create a KK position with the same side to move and board configuration (King positions) but with a different distribution of pieces in hand. By analogous move sequences, other types of pieces can also be transferred from Black to White, and conversely from White to Black. Therefore, all KK positions with Black's King on 5h and White's King on 5b are mutually reachable.
\end{proof}

\begin{lemma}\label{lemma:between_kk}
	Any two KK positions $(t, \bm{b}, \bm{h})$ and $(t', \bm{b'}, \bm{h'})$ are mutually reachable.
\end{lemma}
\begin{proof}
	By \cref{lemma:between_area2}, any KK position $(t, \bm{b}, \bm{h})$ is mutually reachable with a KK position having pieces in hand $h$, Black's King on 5h, and White's King on 5b. Furthermore, by \cref{lemma:kk5258}, this is mutually reachable with a KK position having pieces in hand $h'$, Black's King on 5h, and White's King on 5b. Applying \cref{lemma:between_area2} again, we conclude mutual reachability with the KK position $(t', \bm{b'}, \bm{h'})$.
\end{proof}

\begin{table}[tb]
	\caption{An example of a move sequence from the initial position to a KK position}
	\label{tab:init_to_kk}
	\begin{tabular}{|llll|}
		\hline
		\sente P-7f  & \gote R-3b  & \sente Bx3c= & \gote Rx3c            \\
		\sente R-3h  & \gote Rx3g= & \sente Rx3g  & \gote +Bx9i           \\
		\sente Rx3a= & \gote +Bx8i & \sente Rx2a= & \gote +Bx6g           \\
		\sente Rx2c= & \gote +Bx7f & \sente Rx1c= & \gote +Bx8g           \\
		\sente Rx4c= & \gote Lx1g= & \sente Lx1g  & \gote +Bx9g           \\
		\sente Rx5c= & \gote K-4b  & \sente Rx6c= & \gote +Bx7i           \\
		\sente Rx7c= & \gote +Bx5g & \sente Rx7a= & \gote +Bx3i           \\
		\sente Rx8a= & \gote +Bx1g & \sente Rx8c= & \gote +Bx2g           \\
		\sente G-3h  & \gote +Bx3h & \sente Rx9c= & \gote +Bx2i           \\
		\sente Rx9a= & \gote +Bx4g & \sente Rx6a= & \gote +Bx6i           \\
		\sente Kx6i  & \gote K-5b  & \sente Rx4a= & \gote Kx4a (44 plies) \\
		\hline
	\end{tabular}
\end{table}

\begin{thm}\label{thm:initial-kk-each-other}
	Every KK position is reachable from the initial position, and the initial position is reachable from every KK position.
\end{thm}
\begin{proof}
	From the initial position, a KK position can be reached via the move sequence in \tabref{tab:init_to_kk}\footnote{The move sequence in \tabref{tab:init_to_kk} was found by Weighted A* search from the initial position. There exist countless other move sequences from the initial position to a KK position. Furthermore, determining the shortest move sequence from the initial position to a KK position remains an open problem.}. Furthermore, consider a KK position with Black to move, in which Black's King on 5i, White's King on 5a, and the pieces in hand are distributed identically between the two players. From this position, the initial position can be reached by dropping pieces onto their initial squares in an appropriate order. By \cref{lemma:between_kk}, all KK positions are mutually reachable, thus the initial position is mutually reachable with every KK position.
\end{proof}

\end{document}